\documentclass[poms]{poms1}
\usepackage{amsmath}
\usepackage{algorithm}
\usepackage{algpseudocode}
\usepackage{graphicx}
\usepackage{caption}
\usepackage{subcaption}
\usepackage{hhline}
\usepackage{epsfig}%
\usepackage{amsfonts}%
\usepackage{amssymb}
\usepackage{multirow}
\usepackage{enumerate}
\usepackage{siunitx}
\usepackage{flafter}
\usepackage{bbold}
\usepackage{amssymb}
\usepackage{bbm}
\usepackage{natbib}
 \bibpunct[, ]{(}{)}{,}{a}{}{,}%
 \def\bibfont{\small}%
\usepackage{mathtools}
\usepackage[toc,page]{appendix}
\usepackage[color=yellow]{todonotes}
\usepackage{hyperref}
\usepackage{xcolor}
\hypersetup{colorlinks=true,pdfnewwindow=true,pdfstartview=FitH,%
pdftitle="",pdfauthor="",%
urlcolor=blue!90!red!45!black,citecolor=blue!90!red!45!black,linkcolor=red!90!black}
\usepackage{color}
\newcommand{\proj}{\text{Proj}}
 \renewcommand{\tilde}{\widetilde}
\definecolor{mygray}{RGB}{195,195,195}

\OneAndAHalfSpacedXI 

\usepackage{natbib}
 \bibpunct[, ]{(}{)}{,}{a}{}{,}%
 \def\bibfont{\small}%
\TheoremsNumberedThrough     
\ECRepeatTheorems

\EquationsNumberedThrough    

\begin{document}
\RUNTITLE{Pricing with Clustering}

\TITLE{\Large Context-Based Dynamic Pricing 
with Online Clustering}




\ARTICLEAUTHORS{%
	\AUTHOR{Sentao Miao\thanks{Corresponding Author}}
	\AFF{Desautels Faculty of Management, McGill University, Montreal, QC H3A 1G5, \EMAIL{sentao.miao@mcgill.ca}} 
	\AUTHOR{Xi Chen}
	\AFF{Leonard N. Stern School of Business, New York University, New York City, NY 10012, \EMAIL{xc13@stern.nyu.edu}}
	\AUTHOR{Xiuli Chao}
	\AFF{Department of Industrial and Operations Engineering, University of Michigan, Ann Arbor, MI, and Supply Chain Optimization Technology (SCOT), Amazon, Seattle, WA,  \EMAIL{xchao@umich.edu}}
	\AUTHOR{Jiaxi Liu,Yidong Zhang}
	\AFF{Alibaba Group, Hangzhou, China 311121,  \EMAIL{galiliu.ljx@alibaba-inc.com}, \EMAIL{tanfu.zyd@alibaba-inc.com}}
} 

\ABSTRACT{%
We consider a context-based dynamic pricing problem of online products, which have low sales. 
Sales data  from Alibaba, a major global online retailer, illustrate the prevalence of low-sale products.  For these products, existing single-product dynamic pricing algorithms do not work well due to insufficient data samples. 
To address this challenge, we propose pricing policies
   that concurrently perform clustering over {\it product demand} and set individual pricing 
   decisions on the fly. By clustering data and identifying products that have similar demand patterns, we utilize sales data from  products within the same cluster to improve demand estimation 
   for better pricing decisions.
  We evaluate the algorithms using regret,
  and the result shows that when product demand functions come from multiple clusters, our algorithms significantly outperform traditional single-product pricing policies. Numerical experiments using a real dataset from Alibaba demonstrate that the proposed policies, compared with several benchmark policies, increase the revenue.  
The results show that 
   online clustering is an effective approach to tackling dynamic pricing problems associated with low-sale products.
}
 \HISTORY{Submitted July 2021; revised March 2022; accepted May 2022.}
\maketitle

\baselineskip =18.5pt

\section{Introduction}\label{sec:intro}
Over the past several decades, dynamic pricing has been widely adopted by 
industries, such as retail, airlines, and hotels, with great success (see, e.g., \citealt{smith1992yield,cross1995introduction}). 
Dynamic pricing has been recognized as an important lever  
not only for balancing supply and demand, but also for  increasing revenue and profit. 
Recent advances in online retailing and increased 
 availability of online sales data have created opportunities for firms to better use
 customer information to make pricing decisions,
 see e.g., the survey paper by \cite{den2015dynamic}. Indeed, the advances in information technology have made the sales data easily accessible, facilitating the estimation of demand and the adjustment of price in real time. 
 Increasing availability of demand data allows for
 more knowledge to be gained about the market and customers, as well as the use of advanced analytics tools to make better pricing decisions.
 
 However,  in practice, there are often products with low sales amount or user views. For these products, 
 few available data points exist. For example, 
 \textit{Tmall Supermarket}, a business division of Alibaba, 
 is a {large-scale online 
 store}. 
  In contrast to a typical consumer-to-consumer (C2C) platform (e.g., Taobao under Alibaba) that has millions of products available, Tmall Supermarket is designed to provide carefully selected high-quality products to customers.   We reviewed the sales data from May to July of 2018 on Tmall Supermarket with  nearly 75,000 products offered during this period of time, and it shows that more than 16,000 products ($21.6\%$ of all products) have a daily average number of unique visitors\footnote{A terminology used within Alibaba to represent a unique user login identification.} less than 10, and more than 10,000 products ($14.3\%$ of all products) have a daily average number of unique visitors at most 1. Although each low-sale product alone may have little impact on the company's revenue, the combined sales of all low-sale products are significant.

Pricing low-sale products is often challenging due to the limited sales records available for demand estimation. In fast-evolving markets (e.g., fashion or online advertising), demand data from the distant past may not be useful for predicting customers' purchasing behavior in the near future.  
Classical statistical estimation theory has shown that data insufficiency leads to large estimation error of the underlying demand, which results in sub-optimal pricing decisions. { To overcome the issue of data insufficiency, some literature in other applications such as customer segmentation cluster the objects (i.e., customers) using their feature data (see e.g., \citealt{su2015method}). However, clustering low-sale products by their features may not always work well. 
{ For the following two reasons, in this paper we choose to define clusters based on demand patterns rather than features for low-sale products: i) Some products with very similar features may have very different demand, e.g., two bags with same/similar appearance may have totally different demand because one belongs to a famous brand and the other is a copycat (only difference is the feature of product brand). ii) Some (seemingly) completely unrelated products may exhibit same or similar demand pattern.  }
}
In fact,  the research on dynamic pricing of products with little sales data remains relatively unexplored. To the best of our knowledge, 
 there exists no dynamic pricing policy  in the literature for low-sale products that admits theoretical performance guarantee. This paper fills the gap by developing adaptive context-based dynamic pricing learning  algorithms for low-sale products, and our results show that the algorithms
 perform well both theoretically and numerically.

\subsection{Contributions of this paper}\label{subsec:contribution}

Although each low-sale product only has a few sales records, the total number of low-sale products is usually quite large.
In this paper, we address the challenge of pricing low-sale products   using an important idea from machine learning ---  clustering. 
 Our starting point is that there are some set of 
 products out there, though we do not know which ones, 
 that share similar underlying demand patterns. For these products, information
 can be extracted from their collective sales data to improve the estimation of their demand function. The problem is formulated as developing adaptive learning algorithms that identify the products exhibiting similar demand patterns, and extract 
 the hidden information from sales data of seemingly unrelated products to  
 improve the pricing decisions of low-sale products and increase revenue. { As we mentioned earlier in the introduction, our method of clustering is based on similar demand patterns instead of similar product features. The reason is that products with similar features may have different demand. }

We consider a generalized linear demand model with arbitrary contextual covariate information about products and develop a learning algorithm
that integrates product clustering with pricing decisions.
{  Our policy consists of two phases. The first phase constructs confidence bounds on the distance between clusters, which enables dynamic clustering without any prior knowledge of the cluster structure. The second phase carefully controls the price variation based on the estimated clusters,  striking a proper balance between price exploration and revenue maximization by exploiting the cluster structure.}
Since the pricing part of the algorithm is
inspired by semi-myopic policy  proposed by \citet{keskin2014dynamic}, we refer to
our algorithm as  the {\it Clustered Semi-Myopic Pricing} (CSMP) policy.
We first establish the theoretical regret bound of the proposed policy. Specifically, when the demand functions of the products belong to $m$ clusters, where $m$ is smaller than the total number of products (denoted by $n$), the performance of our algorithm is better than that of  existing dynamic pricing policies that treat each product separately.  
{Let $T$ denote the length of the selling season;}  we show in Theorem \ref{thm:main} that our algorithm achieves the regret of $\tilde O(\sqrt{mT})$, where $\tilde O(\cdot)$ hides the logarithmic terms. This result, when $m$ is much smaller than $n$, is a significant improvement over the regret when applying a single-product pricing policy to individual products, which is typically 
$\tilde O(\sqrt{nT})$.



We carry out a thorough numerical experiment using both synthetic data and a real dataset from Alibaba consisting of a large number of low-sale products. Several benchmarks, one treats each product separately, one puts all products into a single cluster, and the other one applies a classical clustering method ($K$-means method for illustration), are compared with our algorithms under various scenarios. The numerical results show that our algorithms are effective and their performances are consistent in different scenarios (e.g., with almost static covariates, model misspecification).



It is well-known that providing a performance guarantee for a clustering method is challenging 
due to the non-convexity of the loss function (e.g., in $K$-means), which is why there exists no clustering and pricing policy with theoretical guarantees in the existing literature.
This is the first paper to establish the regret bound for a dynamic clustering and pricing policy.
 Instead of adopting an existing clustering algorithm from the machine learning literature (e.g., $K$-means), which usually requires the number of clusters as an input, our algorithms dynamically update the clusters based on the gathered information about customers' purchase behavior. 
In addition to significantly improving the theoretical 
performance as compared to 
classical dynamic pricing algorithms without clustering, our 
algorithms demonstrate excellent 
performance in our simulation study.


\subsection{Literature review}\label{subsec:lit_review}
In this subsection, we review some related research from both the revenue management and machine learning literature.

\textbf{Related literature in dynamic pricing. }
Due to increasing popularity of online retailing, dynamic pricing has become an active research area in revenue management in the past decade.  We only briefly review a few of the most related works and refer the interested readers to \citet{den2015dynamic,kumar2018research} for comprehensive literature surveys. Earlier work and review of dynamic pricing include \citet{gallego1994optimal,gallego1997multiproduct,bitran2003overview,elmaghraby2003dynamic}.
These papers assume that demand information is known to the retailer \textit{{\it a priori}} and either characterize or compute the optimal pricing decisions. In some retailing industries, such as fast fashion, this assumption may not hold due to the quickly changing market environment. As a result, with the recent development of information technology, combining dynamic pricing with demand learning has attracted much interest in research. Depending on the structure of the underlying demand functions, these works can be roughly divided into two categories: 
 parametric demand models (see, e.g., \citealt{carvalho2005learning,bertsimas2006dynamic,besbes2009dynamic,farias2010dynamic,broder2012dynamic,harrison2012bayesian,den2013simultaneously,keskin2014dynamic,wang2021uncertainty,Chen:22:privacy}) and nonparametric demand models (see, e.g., \citealt{araman2009dynamic,wang2014close,lei2014near,chen2015real,besbes2015surprising,cheung2017dynamic,cohen2018dynamic,chen2019network,chen2021differential, Chen:22:Robust}).  The aforementioned papers assume that the price is continuous. Other works consider a discrete set of prices, see, e.g., \cite{ferreira2018online}, and recent studies examine pricing problems with strategic customers (e.g., \citet{Chen:22:Strategic}) or in dynamically changing  environments (e.g.,  \citet{besbes2015non} and \cite{keskin2016chasing})

Dynamic pricing and learning with demand covariates (or contextual information) has received increasing attention in recent years 
 because of its flexibility and clarity in modeling customers and market environment. Research involving this information include, among others, \citet{chen2015statistical,qiang2016dynamic,nambiar2016dynamic,ban2017personalized,lobel2018multidimensional,chen2018nonparametric,javanmard2016dynamic}.
In many online-retailing applications, sellers have access to rich covariate information reflecting the current market situation. Moreover, the covariate information is not static but usually evolves over time. Our paper  incorporates time-evolving covariate information into the demand model. In particular, given the observable covariate information of a product, we assume that the customer decision depends on both the selling price and covariates. 
Although covariates provide richer information for accurate  demand estimation, a demand model that incorporates covariate information involves more parameters to be estimated. {
Therefore, it requires more data for estimation with the presence of covariates, which poses an additional challenge for low-sale products.}
	
{\bf Related literature in clustering for pricing.} In the literature, there are some interesting two-stage methods that first use historical data to determine the cluster structure of demand functions in an offline manner, and then dynamically make
pricing decisions for another product by learning which cluster its demand  belongs to. \citet{ferreira2015analytics} study a pricing problem with flash sales on the Rue La La platform.
Using historical information and offline optimization, the authors classify the demand of all products into multiple groups, and use demand information for products that did not experience lost sales to estimate demand for 
products that had lost sales. They construct ``demand curves'' on the percentage of total sales with respect to the number of hours after the sales event starts, then classify these curves into four  clusters. For a sold-out product, they check which one of the four curves is the closest to its sales behavior and use that to estimate the lost sales. 
\citet{cheung2017dynamic} consider the single-product pricing problem, where the demand of the product is assumed to be from one of the $K$ demand functions (called \textit{demand hypothesis} in that paper). Those $K$ demand functions are assumed to be known, and the decision is to choose which  of those functions is the true demand curve of the product. In their field experiment with Groupon, they applied $K$-means clustering to historical demand data to generate those $K$ demand functions offline. That is, clustering is conducted offline first using historical data, then dynamic pricing decisions are made in an online fashion for a new product, assuming that  
its demand is one of the $K$ demand functions.  { Very recently, \citet{keskin2020data} studied personalized pricing in retail electricity market, where they applied a spectral clustering approach to decide customer types based on customers' features. In this paper, as we discussed earlier, instead of using observable features, we cluster products based on estimated demand patterns.}


{\label{mark:lit_other_OM}
\textbf{Related literature in other operations management problems.}
The method of clustering is quite popular for many operations management problems such as demand forecast for new products and customer segmentation. In the following, we give a brief review of some recent papers on these two topics that are based on data clustering approach. 

Demand forecasting for new products is a prevalent yet challenging problem.
Since new products at launch have no historical sales data, a commonly used approach is to borrow data from ``similar old products'' for demand forecasting. To connect the new product with old products, current literature typically use product features. 
For instance, \citet{baardman2017leveraging} assume a demand function which is a weighted sum of unknown functions (each representing a cluster) of product features. While in \citet{ban2018dynamic}, similar products are predefined such that common demand parameters are estimated using sales data of old products. \citet{hu2018forecasting} investigate the effectiveness of clustering based on product category, features, or time series of demand respectively.


Customer segmentation is another application of clustering. \citet{jagabathula2018model} assume a general parametric model for customers' features with unknown parameters, and use $K$-means clustering to segment customers. \citet{bernstein2018dynamic} consider the dynamic personalized assortment optimization using clustering of customers. They develop a hierarchical Bayesian model for mapping from customer profiles to segments.

}

{
Compared with these literature, besides a totally different problem setting, our paper is also different in the approach. First, we consider an online clustering approach with provable performance instead of an offline setting as in \citet{baardman2017leveraging,ban2018dynamic,hu2018forecasting,jagabathula2018model}. Second, we know neither the number of clusters (in contrast to \citealt{baardman2017leveraging,bernstein2018dynamic} that assume known number of clusters), nor the set of products in each cluster (as compared with \citealt{ban2018dynamic} who assume known products in each cluster). 
Finally, we do not assume any specific probabilistic structure on the demand model and clusters (in contrast with \citealt{bernstein2018dynamic} who 
assign and update the probability for a product to belong to some cluster),
but define clusters using product neighborhood based on their estimated demand  parameters. 
}

\textbf{Related literature in multi-arm bandit problem. } A successful dynamic pricing algorithm requires a careful balancing  between exploration (i.e., learning the underlying demand function) and exploitation (i.e., making the optimal pricing strategy based on the learned information so far). The exploration-exploitation trade-off has been extensively investigated in the multi-armed bandit (MAB) literature; see 
\citet{bubeck2012regret} for a comprehensive literature review. Among the vast MAB literature, there is a line of research on bandit clustering that addresses a different but related problem (see, e.g., \citealt{cesa2013gang,gentile2014online,nguyen2014dynamic,gentile2016context}). The setting is that there is a finite number of arms which belong to several unknown clusters, where  unknown reward functions of arms in each cluster are the same. Under this assumption, the MAB algorithms aim to cluster different arms and learn the reward function for each cluster.
The setting of the bandit-clustering problem is quite different from ours. In the bandit clustering problem, the arms belong to different
clusters and the decision for each period is which arm to play. In our setting, the 
products belong to different clusters and the decision for each period is what prices to charge for all products, and  
we have a \emph{continuum set} of prices to choose from for each product. In addition, 
in contrast to the linear reward in bandit-clustering problem, 
the demand functions in our setting follow a generalized linear model.
As will be seen in Section \ref{sec:policy_and_results}, we design a price perturbation strategy based on the estimated cluster, which is very different from the  algorithms in  bandit-clustering literature. 

{\label{mark:lit_clustering}
\textbf{Related literature in clustering.} 
We end this section by giving a brief overview of clustering methods in the machine learning literature.  To save space, 
we only discuss several popular clustering methods, and refer 
the interested reader to \citet{saxena2017review} for 
a recent literature review on the topic.
The first one is called hierarchical clustering \citep{murtagh1983survey}, which iteratively clusters objects (either bottom-up, from a single object to several big clusters; or top-down, from a big cluster to single product). Comparable with hierarchical clustering, another class of clustering method is partitional clustering, in which the objects do not have any hierarchical structure, but rather are grouped into different clusters horizontally. Among these clustering methods, $K$-means clustering is probably the most well-known and most widely applied method (see e.g., \citealt{macqueen1967some}).
Several extensions and modifications of $K$-means clustering method have been
proposed in the literature, e.g., 
 $K$-means++ \citep{arthur2007k} and  fuzzy c-means clustering \citep{dunn1973fuzzy}. Another important class of clustering method is based on graph theory. For instance, the spectral clustering uses graph Laplacian to help determine clusters \citep{shi2000normalized,von2007tutorial}. 
Beside these general methods for clustering, there are many clustering methods for specific problems such as decision tree, neural network, etc. 
It should be noted that nearly all the clustering methods in the literature are based on offline data. This paper, however, integrates clustering into online learning and decision-making process.

}

\subsection{Organization of the paper}
The remainder of this paper is organized as follows. In Section \ref{sec:model}, we present the problem formulation. Our main algorithm is presented in Section \ref{sec:policy_and_results} together with the theoretical results for the algorithm performance. 
In Section \ref{sec:simulation}, we report the results of several numerical experiments based on both synthetic data and a real dataset.
We conclude the paper with a discussion about future research in Section \ref{sec:conclusion}. Finally, all the technical proofs are presented in the  supplement. 

\section{Problem Formulation}\label{sec:model}
We consider a retailer that sells $n$ products, labeled by $i=1,2,\ldots, n$,
with unlimited inventory (e.g., there is an inventory replenishment scheme 
such that products typically do not run out of stock). Following the literature, we denote the set of these products by $[n]$. 
{
We mainly focus on online retailing of low-sale products. These products are typically not offered to customers as a display; 
hence we do not consider substitutability/complementarity of products in our model. Furthermore, these products are usually not recommended by the retailer on the
platform, and instead, customers search to view them online. We let $q_i>0$ denote the percentage of potential customers who are interested in, or view/search for, product $i\in [n]$. 
In this paper, we will treat $q_i$ as the probability an arriving customer views product $i$; in another word, $q_i$ can be considered as the arrival rate of customers viewing product $i$ which are independent from each other.
}

{ Customers arrive sequentially at time $t=1,2,\ldots, T$, and we denote the set of all time indices by $[T]$. For simplicity, we assume without loss of generality that there is exactly one arrival during each period. In each time period $t$, the firm first observes some covariates for each product $i$,
such as product rating, prices of competitors, average sales in past few weeks,
and promotion-related information (e.g., whether the product is currently on sale).
We denote the covariates of product $i$ by $z_{i,t}\in\mathbb{R}^d$, 
where $d$ is the dimension of the covariates that is usually small (as compared to $n$ or $T$). 
The covariates $z_{i,t}$ change over time and satisfy $||z_{i,t}||_2\leq 1$ after normalization. Then, the retailer sets the price $p_{i,t}\in [\underline p,\overline p]$ for each product $i$, where $0\leq\underline p<\overline p<\infty$ (the assumption of the same price range for all products is without loss of generality).   Let $i_t$ denote the product that the customer searches in period $t$ (or customer $t$). 
{
After observing the price and other details of product $i_t$, 
customer $t$ then decides whether or not to purchase it. 
The sequence of events in period $t$ is summarized as follows: 
\smallskip

\begin{enumerate}
    \item [i)] In time $t$, the retailer observes the covariates $z_{i,t}$ for each product $i\in [n]$, then sets the price $p_{i,t}$ for each $i\in [n]$.
    
    \item [ii)] Customer searches for product $i_t \in [n]$ in period $t$ with probability $q_{i_t}$ independent of others and then observes its price $p_{i_t,t}$.
    
    \item [iii)] The customer decides whether or not to purchase product $i_t$.
\end{enumerate}
}

\medskip

}

The customer's purchasing decision follows a 
\textit{generalized linear model} (GLM, see e.g., \citealt{mccullagh1989generalized}). 
That is, given price $p_{i_t,t}$ of product $i_t$ at time $t$, the customer's purchase decision is represented  by a Bernoulli random variable $d_{i_t, t}(p_{i_t,t};z_{i_t,t})\in \{0,1\}$, where $d_{i_t, t}(p_{i_t,t};z_{i_t,t})=1$ if the customer purchases product $i_t$ and 0 otherwise. The purchase probability, which is the expectation of $d_{i_t,t}(p_{i_t,t};z_{i_t,t})$, takes the form
\begin{equation}\label{eq:demand_model}
\mathbb{E}[d_{i_t,t}(p_{i_t,t};z_{i_t,t})]=\mu(\alpha_{i_t}'x_{i_t,t}+\beta_{i_t}p_{i_t,t}),
\end{equation}
where $\mu(\cdot)$ is the link function,  $x_{i_t,t}'=(1,z_{i_t,t}')$ is the corresponding extended demand covariate with the 1 in  the first entry used to model the bias term in a GLM model, and the expectation is taken with respect to customer purchasing decision. 
Let $\theta_{i_t}'=(\alpha_{i_t}',\beta_{i_t})$ be the unknown parameter of product $i_t$,
which is assumed to be bounded. That is, $||\theta_i||_2\leq L$ for some constant $L$ for all $i\in [n]$. 

\medskip
\begin{remark}
The commonly used linear and logistic models are special cases of GLM with link function $\mu(x)=x$ and $\mu(x)=\exp{(x)}/(1+\exp(x))$, respectively. The parametric demand model \eqref{eq:demand_model} has been used in a number of papers on pricing with contextual information, see, e.g., \cite{qiang2016dynamic} 
(for a special case of linear demand with $\mu(x)=x$)
and \cite{ban2017personalized}.
\end{remark}

\medskip

For convenience and with a slight abuse of notation, we write  
$$p_t:=p_{i_t,t}, \;\; z_t:=z_{i_t,t}, \;\; x_t:=x_{i_t,t},\;\;d_{t}:=d_{i_t,t}, 
$$
where $``:="$ stands for ``defined as". 
Let the feasible sets of $x_t$ and $\theta_i$ be denoted as $\mathcal{X}$ and $\Theta$,
respectively. We further define
\begin{equation}\label{eq:def_T}
\mathcal{T}_{i,t}:=\{s\leq t\;: i_s=i \}
\end{equation}
as the set of time periods before $t$ in which product $i$ is viewed, and $T_{i,t}:=|\mathcal{T}_{i,t}|$ its cardinality. With this demand model, the expected revenue $r_t(p_t)$ of each round $t$ is \begin{eqnarray}
r_t(p_t):=p_t\mu(\alpha_{i_t}'x_{t}+\beta_{i_t}p_{t}).
\label{revenue}
\end{eqnarray}
Note that we have made the dependency of $r_t(p_t)$ on $x_t$ implicit. 

\textbf{The firm's optimization problem and regret.}  The firm's goal is to decide the price $p_t\in [\underline{p},\overline{p}]$ at each time $t$ for each product to maximize the cumulative expected revenue $\sum_{t=1}^{T}\mathbb{E}[r_t(p_t)]$, where the expectation is taken with respect to the randomness of the pricing policy as well as the stream of $i_t$ for $t\in [T]$, and for the
next section, also the stochasticity in contextual covariates $z_t$, $t\in [T]$.  The goal of  maximizing the expected cumulative revenue is equivalent to minimizing the so-called regret, which is defined as the revenue gap as compared with 
the \textit{clairvoyant decision maker} who knew the underlying parameters in the demand model {\it a priori}.  With the known demand model, the optimal price can be computed as
\[
p_t^*=\argmax_{p\in [\underline{p},\overline{p}]} r_t(p),
\]
and the corresponding revenue gap at time $t$ is  $\mathbb{E}[r_t(p_t^*)-r_t(p_t)]$ (the dependency of
$p_t^*$ on $x_t$ is again made implicit). The cumulative regret of a policy 
$\pi$ with prices $\{p_t\}_{t=1}^T\;\;$ is defined by the summation of revenue gaps over the entire time horizon, i.e.,
\begin{equation}\label{eq:def_regret}
R^{\pi}(T):=\sum_{t=1}^{T}\mathbb{E}[r_t(p_t^*)-r_t(p_t)].
\end{equation}

\begin{remark}
For consistency with the online pricing literature, see 
e.g., \cite{chen2015statistical,qiang2016dynamic,ban2017personalized,javanmard2016dynamic}, in this paper we use expected revenue as the objective to maximize. However, we point out that all our analyses and results carry over to the objective of profit maximization. That is, if $c_{t}$ is the cost of the product in round $t$, then the expected profit in (\ref{revenue}) can be replaced by 
$$r_t(p_t)=(p_{t}-c_{t})\mu(\alpha_{i_t}'x_{t}+\beta_{i_t}p_{t}).$$

\end{remark}

\textbf{Cluster of products.}  
Two products $i_1$ and $i_2$ are said to be ``similar'' 
if they have  similar underlying demand functions, i.e., $\theta_{i_1}$ and $\theta_{i_2}$ are close.   In this paper we assume that the $n$ products can be partitioned into $m$ clusters,
$\mathcal{N}_j$ for $j=1,2,\ldots,m$, 
such that for arbitrary two products $i_1$ and $i_2$, we have $\theta_{i_1}=\theta_{i_2}$ if $i_1$ and $i_2$ belong to the same cluster; otherwise, $||\theta_{i_1}-\theta_{i_2}||_2\geq \gamma>0$ for some constant $\gamma$. We refer to this cluster structure as the $\gamma$-gap assumption, which will be relaxed in Remark \ref{cluster} of Section \ref{subsec:theory}.
For convenience, we denote the set of clusters by $[m]$, and by a bit abuse of notation, let $\mathcal{N}_i$ be the cluster to which
product $i$ belongs. 

It is important to note that the number of clusters $m$ and each cluster $\mathcal{N}_j$ are {\it unknown} to the decision maker {\it a priori}. Indeed, in some applications 
such structure may not exist at all. 
If such structure does exist, then our policy can identify such a cluster structure and make use of it to improve the practical performance and the regret bound. 
However, we point out that the cluster structure is not a requirement for the 
 pricing policy to be discussed.
In other words, our policy reduces to a
standard dynamic pricing algorithm when demand functions of the products are all different (i.e., when $m=n$).

{\label{mark:clustering}
It is also worthwhile to note that our clustering  is based on demand parameters/patterns and {\it not} on product categories or features, since
it is the demand of the products that we want to learn. The clustering 
approach based on demand is prevalent in the literature (besides \citealt{ferreira2015analytics,cheung2017dynamic} and the references therein, we also refer to \citealt{van2012sku} for a comprehensive review). Clustering based on category/feature similarity is useful in some problems (see e.g., \citealt{su2015method} investigate customer segmentation using features of clicking data), but it does 
not apply to our setting because, for instance, products with similar feature for different brands may have very different demand (see our earlier discussion in the introduction). {Moreover, in our model, motivated by Alibaba's business the product feature $x_{i,t}$ is non-stationary, so feature-based clustering can lead to different clusters in different time.
}}

\begin{remark}
For its application to the online pricing problem, 
the contextual information in our model is about the product. 
That is, at the beginning of each period, the firm observes the
contextual information about each product, then determines 
the pricing decision for the product, and then the arriving customer makes a purchasing decisions. We point out that our algorithm and result apply equally to
personalized pricing in which the contextual information is about the customer. That is, a customer arrives (e.g., logging on the website) and reveals his/her contextual information, and then the firm makes a pricing decision based on that information. The objective 
is to make personalized pricing decisions to
 maximize total revenue (see e.g., \citealt{ban2017personalized}). 
\end{remark}

\section{Pricing Policy and Main Results}\label{sec:policy_and_results}
In this section we discuss the specifics of the learning algorithm, its theoretical performance, and a sketch of its proof. Specifically, we describe the policy procedure and discuss its intuitions in Section \ref{subsec:description} before presenting its regret and outlining the proof in Section \ref{subsec:theory}. 

\subsection{Description of the pricing policy}\label{subsec:description}
 Our policy consists of two phases for each period $t\in [T]$:
 the first phase constructs a  \textit{neighborhood} for each product $i\in [n]$, and the second phase determines its selling price. 
In the first step, our policy uses \textit{individual data} of each product $i\in[n]$ 
to estimate parameters $\hat{\theta}_{i,t-1}$. This estimation is used only for construction of the neighborhood $\hat{\mathcal{N}}_{i,t}$ for product $i$. Once the neighborhood is defined, we consider all the products in this neighborhood as in the same cluster and use \textit{clustered data} to estimate the parameter vector $\widetilde{\theta}_{\hat{\mathcal{N}}_{i,t},t-1}$. The latter is used in computing the selling price of product $i$. We refer to Figure \ref{fig:flow} for a flowchart of our policy, and present the detailed procedure in Algorithm \ref{alg:CSMP}.

In the following, we discuss the parameter estimation of GLM demand functions and the construction of a  neighborhood in detail.

\begin{figure}[!h]
\vspace{0.15in}
\centering
\includegraphics[width=0.8\textwidth]{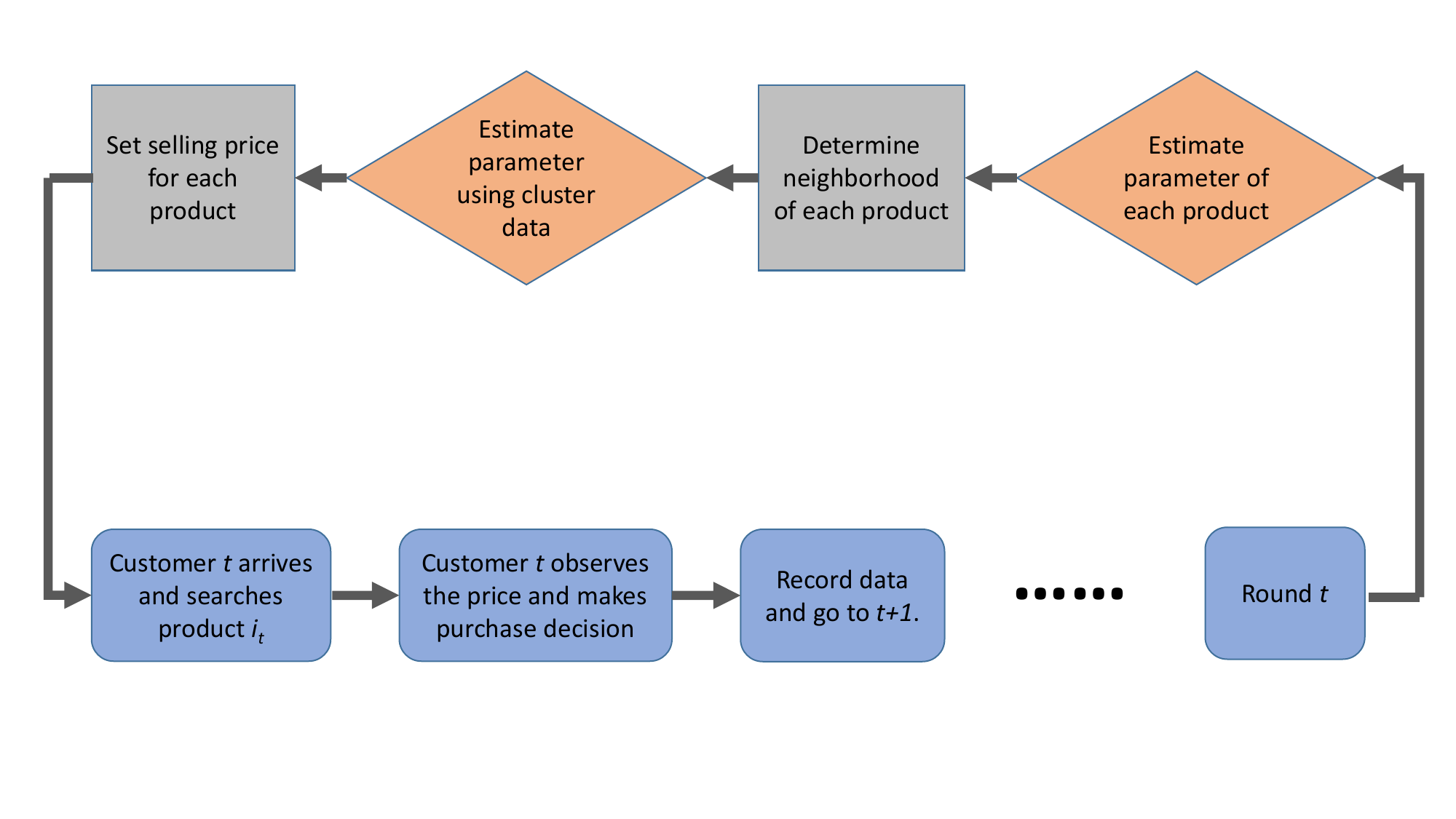}
\caption{Flow chart of the algorithm.}\label{fig:flow}
\vspace{.1in}
\end{figure}

\textbf{Parameter estimation of GLM.} As shown in Figure \ref{fig:flow}, the parameter estimation 
is an important part of our policy construction. We adopt the classical maximum likelihood estimation (MLE) method for parameter estimation (see \citealt{mccullagh1989generalized}). 
For completeness, we briefly describe the MLE method here. Let $u_t:=(x_t',p_t)' \in \mathbb{R}^{d+2}$. The conditional distribution of the demand realization $d_t$, given $u_t$, belongs to the exponential family and can be written as
\begin{equation}\label{eq:exp_family}
\mathbb{P}(d_t|u_t)=\exp\left(\frac{d_t u_t'\theta-m(u_t'\theta) }{g(\eta)}+h(d_t,\eta) \right).
\end{equation}
Here $m(\cdot),\;g(\cdot)$, and $h(\cdot)$ are some specific functions, where $\dot{m}(u_t'\theta)=\mathbb{E}[d_t]=\mu(u_t'\theta)$ depends on $\mu(\cdot)$ and  $h(d_t, \eta)$ is the normalization part, and $\eta$ is some known scale parameter. Suppose that we have $t$ samples $(d_s,p_s)$ for $s=1,2,\ldots,t$, the negative log-likelihood function of $\theta$ under model (\ref{eq:exp_family}) is
\begin{equation}\label{eq:likelihood}
\sum_{s=1}^{t}\left(\frac{m(u_s'\theta)-d_s u_s'\theta }{g(\eta)}+h(d_s,\eta)\right).
\end{equation}
By extracting the terms in \eqref{eq:likelihood} that involves $\theta$,  the maximum likelihood estimator $\hat{\theta}$ is
\begin{equation}\label{eq:pure_MLE}
\hat{\theta}=\argmin_{\theta\in \Theta}\sum_{s=1}^{t}l_s(\theta),  \qquad  l_s(\theta):=m(u_s'\theta)-d_s u_s'\theta. 
\end{equation} 
Since $\nabla^2 l_s(\theta)=\dot{\mu}(u_s'\theta)u_su_s'$ is positive semi-definite in a standard GLM model (by Assumption A-\ref{assumption:mu} in the next subsection), the optimization problem in (\ref{eq:pure_MLE}) is convex and can be easily solved.

\textbf{Determining the neighborhood of each product.} 
The first phase of our policy determines which products to include in the neighborhood of each product $i\in [n]$. We use the term ``neighborhood'' instead of cluster, though closely related, because clusters are usually assumed to be disjoint in the machine learning literature. In contrast, by our definition of neighborhood, some products can belong to different neighborhoods depending on the estimated parameters. To define the neighborhood of $i$, which is denoted by $\hat{\mathcal{N}}_{i,t}$, we first estimate parameter $\hat{\theta}_{i,t-1}$ of each product $i\in[n]$ using their own data, i.e.,  $\hat{\theta}_{i,t-1}$ is the maximum likelihood estimator using data in $\mathcal{T}_{i,t-1}$ defined in \eqref{eq:def_T}. Then, we include a product $i'\in[n]$ in the neighborhood $\hat{\mathcal{N}}_{i,t}$ of $i$ if their estimated parameters are \textit{sufficiently close}, which is defined as $$||\hat{\theta}_{i',t-1}-\hat{\theta}_{i,t-1}||_2\leq B_{i',t-1}+B_{i,t-1},$$ where $B_{i,t-1}$ is a \textit{confidence bound} for product $i$ given by
\begin{equation}\label{eq:confidence_bound}
B_{i,t}:=\frac{\sqrt{c(d+2)\log(1+t)}}{\sqrt{\lambda_{\min}({V}_{i,t})}}.
\end{equation}
Here, ${V}_{i,t}:=I+\sum_{s\in\mathcal{T}_{i,t} }u_su_s'$  is the empirical Fisher's information matrix of product $i\in[n]$ at time $t$ and $c$ is some positive constant, which will be specified in our theory development. Note that, 
by the $\gamma$-gap assumption discussed at the end of Section \ref{sec:model}, 
the method will work even when $T_{i,t-1}$ only contains a limited number of sales records. 

\textbf{Setting the price of each product.} Once we define the (estimated) neighborhood $\hat{\mathcal{N}}_{i,t}$ of $i\in[n]$, we can pool the demand data of all products in $\hat{\mathcal{N}}_{i,t}$ to learn the parameter vector. That is,  we let 
\[
\widetilde{\mathcal{T}}_{\hat{\mathcal{N}}_{i,t},t-1}:=\bigcup_{i'\in\hat{\mathcal{N}}_{i,t}}\mathcal{T}_{i',t-1}  \quad \text{and} \quad \tilde T_{\hat{\mathcal{N}}_{i,t},t-1}:=|\widetilde{\mathcal{T}}_{\hat{\mathcal{N}}_{i,t},t-1}|.
\]
The clustered parameter vector $\tilde{\theta}_{\hat{\mathcal{N}}_{i,t},t-1}$ is the maximum likelihood estimator using data in $\tilde{\mathcal{T}}_{\hat{\mathcal{N}}_{i,t},t-1}$.


To decide on the price, we first compute $p_{i,t}'$, which is the ``optimal price'' based on the estimated clustered parameters $\tilde{\theta}_{\hat{\mathcal{N}}_{i,t},t-1}$.
Then we restrict $p_{i,t}'$ to the interval $[\underline{p}+|\Delta_{i,t}|,\overline{p}-|\Delta_{i,t}|]$ by the \emph{projection operator}. That is, we  compute 
$$\tilde{p}_{i,t}=\proj_{[\underline{p}+|\Delta_{i,t}|,\overline{p}-|\Delta_{i,t}|]}(p_{i,t}'), \;\;\hbox { where } \;\; \proj_{[a,b]}(x):=\min\{\max\{x,a\},b\}.$$ 
The reasoning for this restriction is that our final price $p_{i,t}$ will be $p_{i,t}=\tilde{p}_{i,t}+\Delta_{i,t}$, and the projection operator forces the final price $p_{i,t}$ to the range $[\underline p,\overline p]$. Here, the price perturbation $\Delta_{i,t}=\pm\Delta_0 \tilde{T}_{\hat{\mathcal{N}}_{i,t},t}^{-1/4}$ takes a positive or a negative value with equal probability, where $\Delta_0$ is a positive constant. We add this price perturbation for the purpose of price exploration. Intuitively, the more price variation we have, the more accurate the parameter estimation will be. However, too much price variation leads to loss of revenue because we deliberately charged a ``wrong'' price. Therefore, it is crucial to find a balance between these two targets by defining an appropriate $\Delta_{i,t}$. 

We note that this pricing scheme belongs to the class of  semi-myopic pricing policies defined in \cite{keskin2014dynamic}. Since our policy combines clustering with semi-myopic pricing, we refer to it as the \textit{Clustered Semi-Myopic Pricing} (CSMP) algorithm.

\medskip

\begin{algorithm}[h]
\caption{The CSMP Algorithm}
\label{alg:CSMP}
\begin{algorithmic}[1]
\Require
$c$, the confidence bound parameter;
$\Delta_0$, price perturbation parameter;

\State \textbf{Step 0. Initialization. } Initialize $\mathcal{T}_{i,0}=\emptyset$ and $V_{i,0}=I$ for all $i\in[n]$. 
Let $t=1$ and go to Step 1.

\For{$t=1,2,\ldots,T$}

\State \textbf{Step 1. Individual Parametric Estimation.} Compute the MLE using individual data
\[
\hat\theta_{i,t-1}=\argmin_{\theta\in \Theta}\sum_{s\in\mathcal{T}_{i,t-1}}l_s(\theta)
\]
for all $i\in[n]$.  Go to Step 2. 

\State \textbf{Step 2. Neighborhood Construction.} Compute the neighborhood of each product $i$ as
\[
\hat{\mathcal{N}}_{i,t}=\{i'\in[n]:\; ||\hat{\theta}_{i',t-1}-\hat{\theta}_{i,t-1}||_2\leq {B}_{i',t-1}+{B}_{i,t-1} \}
\]
where ${B}_{i,t-1}$ is defined in (\ref{eq:confidence_bound}) for each $i\in[n]$. Go to Step 3. 

\State \textbf{Step 3. Clustered Parametric Estimation.} Compute the MLE using clustered data
\[
(\tilde{\alpha}_{\hat{\mathcal{N}}_{i,t},t-1}',\tilde{\beta}_{\hat{\mathcal{N}}_{i,t},t-1})'=\tilde{\theta}_{\hat{\mathcal{N}}_{i,t},t-1}=\argmin_{\theta\in \Theta}\sum_{s\in\tilde{\mathcal{T}}_{\hat{\mathcal{N}}_{i,t},t-1}}l_s(\theta)
\]
for each $i\in [n]$. Go to Step 4.

\State \textbf{Step 4. Pricing.} Compute price for each $i\in[n]$ as 
\[
p_{i,t}'=\argmax_{p\in [\underline{p},\overline{p}]}\;\mu(\alpha_{\hat{\mathcal{N}}_{i,t},t-1}'x_{i,t}+\beta_{\hat{\mathcal{N}}_{i,t},t-1}p)p,
\]
then project to $\tilde{p}_{i,t}=\proj_{[\underline{p}+|\Delta_{i,t}|,\overline{p}-|\Delta_{i,t}|]}(p_{i,t}')$
and offer to the customer price $p_{i,t}=\tilde p_{i,t}+\Delta_{i,t}$ where $\Delta_{i,t}=\pm\Delta_0 \tilde T_{\hat{\mathcal{N}}_{i,t},t}^{-1/4}$ which takes two signs with equal probability. 

\State Then, customer $t$ arrives, searches for product $i_t$, and makes purchasing
decision $d_{i_t,t}(p_{i_t,t}; z_{i_t,t})$.
Update $\mathcal{T}_{i_t,t}=\mathcal{T}_{i_t,t-1}\cup\{t\}$ and $V_{i_t,t}=V_{i_t,t-1}+u_{t}u_t'$.

\EndFor
\end{algorithmic}
\end{algorithm}

We briefly discuss each step of the algorithm and the intuition behind the theoretical performance. For Steps 1 and 2, the main purpose is to identify the correct neighborhood of the product searched in period $t$; i.e.,
$\hat{\mathcal{N}}_{i_t,t}=\mathcal{N}_{i_t}$ with high probability (for brevity of notation, we let $\hat{\mathcal{N}}_{t}:=\hat{\mathcal{N}}_{i_t,t}$). To achieve that, two conditions are necessary. First, the estimator $\hat\theta_{i,t}$ should converge to $\theta_i$ as $t$ grows for all $i\in [n]$. Second, the confidence bound $B_{i,t}$ should converge to $0$ as $t$ grows, such that in Step 2, we are able to identify different neighborhood by the $\gamma$-gap assumption among clusters. To satisfy these conditions, classical statistical learning theory (see e.g., Lemma \ref{lemma:single_est_error} in the supplement) requires the minimum eigenvalue of the empirical Fisher's information matrix $V_{i,t}$ to be sufficiently above zero, or more specifically, $\lambda_{\min}(V_{i,t})\geq \Omega(q_i\sqrt{t})$ (see Lemma \ref{lemma:emp_fisher_info} in the supplement). This requirement is guaranteed by the variation assumption on demand covariates $x_{i,t}$, which will be imposed in Assumption A-\ref{assumption:stochastic} in 
the next subsection, plus our choice of price perturbation in Step 4. 


Following the discussion above, when $\hat{\mathcal{N}}_t=\mathcal{N}_{i_t}$ with high probability, we can cluster the data within $\mathcal{N}_{i_t}$ to increase the number of samples for $i_t$. Because of the increased data samples, it is expected that the estimator $\tilde\theta_{{\mathcal{N}}_{i_t},t-1}$ for $\theta_{i_t}$ in Step 3 is more accurate than $\hat\theta_{i,t-1}$. Of course, the estimation accuracy again requires the minimum eigenvalue of the empirical Fisher's information matrix over the clustered set $\tilde{\mathcal T}_{\mathcal{N}_{i_t},t-1}$, i.e., $\lambda_{\min}(I+\sum_{s\in\tilde{\mathcal{T}}_{\mathcal{N}_{i_t},t-1}}u_su_s')$,  to be sufficiently large, which is again guaranteed by stochastic assumption of $z_{i,t}$ 
and the price perturbation in Step 4.

The design of the CSMP algorithm
 depends critically on two things. First, by taking an appropriate price perturbation in Step 4, we balance the exploration and exploitation. If the perturbation is too much, even though it helps to achieve good parameter estimation, it may lead to loss of revenue (due to purposely charging the wrong price). Second, the sequence of demand covariates $z_{i,t}$ has to satisfy an important variation assumption (Assumption A-\ref{assumption:stochastic}). Later we will see that this variation assumption is weaker than the typical stochastic assumption on $z_{i,t}$ which is commonly seen in the pricing literature with demand covariates (see e.g., \citealt{chen2015statistical,qiang2016dynamic,ban2017personalized,javanmard2016dynamic}). 

\subsection{Theoretical performance of the CSMP algorithm}\label{subsec:theory}

This section presents the regret of the CSMP pricing policy. 
Before  proceeding to the main result, we first make some technical assumptions that will be needed for the theorem.

\medskip
\noindent\textbf{Assumption A:}\label{assumptions}
\begin{enumerate}
\item\label{assumption:uniq_max} The expected revenue function $p\mu(\alpha'x+\beta p)$ has a unique maximizer $p^*(\alpha'x,\beta)\in [\underline{p},\overline{p}]$, which is Lipschitz in $(\alpha'x,\beta)$ with parameter $L_0$ for all $x\in \mathcal{X}$ and $\theta\in\Theta$. Moreover, the unique maximizer is in the interior $(\underline{p},\overline{p})$ for the true $\theta_i$ for all $i\in[n]$ and $x\in\mathcal{X}$. 

\item\label{assumption:mu} $\mu(\cdot)$ is monotonically increasing and twice continuously differentiable in its feasible region. Moreover, for all $x\in\mathcal{X}$, $\theta\in\Theta$ and $p\in [\underline{p},\overline{p}]$, we have that $\dot{\mu}(\alpha'x+\beta p)\in [l_1,L_1]$, and $|\ddot{\mu}(\alpha'x+\beta p)|\leq L_2$ for some positive constants $l_1,L_1,L_2$.  

{
\item\label{assumption:stochastic}{%
There exist some constants $c_0>0 $ and $t_0>0$, such that for 
any $i\in[n]$ and $t\in [T]$, $\lambda_{\min}(\sum_{s\in\mathcal{T}_{i,t}}x_{i,s}x_{i,s}')\geq c_0 T_{i,t}$ when $T_{i,t}\geq t_0$.}}

\end{enumerate}
\medskip

{
The first assumption A-1 is a standard regularity condition on expected revenue, which is
prevalent in the pricing literature (see e.g., \citealt{broder2012dynamic}). The second assumption A-2 states that the purchasing probability will increase if and only if the utility $\alpha'x+\beta p$ increases, which is plausible.  One can easily verify that the commonly used demand models, such as linear and logistic demand, satisfy these two assumptions with appropriate choice of $\mathcal{X}$ and $\Theta$. 
{%
The last assumption A-3 is a variation assumption on demand covariates. That is, we require the covariates of each product have sufficient variation. Such variation condition is required
for learning in many pricing papers (see e.g.,  \citealt{qiang2016dynamic,ban2017personalized,nambiar2016dynamic,javanmard2016dynamic}). We emphasize that in the literature, to guarantee this variation, $x_{i,t}$ is often assumed to be stochastic (e.g., independent and identically distributed) and $\lambda_{\min}(\mathbb{E}[x_{i,t}x_{i,t}'])$ is strictly positive. With such stochastic assumption, A-3 is satisfied (with high probability and it is sufficient for our result to hold); hence our assumption A-3 is a weaker assumption than the common stochastic assumption in the literature. In our setting, $x_{i,t}$ can be arbitrary and even adversarial as long as sufficient variation is satisfied, and we manage to prove similar theoretical performance of our algorithm under this relaxed assumption (see Section \ref{ap_sec:proof_main} in the online supplement). We note that this relaxed variation assumption is practically more favorable because in reality, the stochastic assumption may be difficult to justify. For instance, {there can be nearly static and nonstochastic features in $z_{i,t}$ (e.g., indicator of weekend/holiday) such that $\lambda_{\min}(\mathbb{E}[z_{i,t}z_{i,t}'])>0$ is violated. We test our algorithm numerically against these cases in Section \ref{subsec:synthetic_simulation}, and the results show
that our algorithm performs well. One might argue that assumption A-3 may still be violated if some features are completely static (such as color, size, and brand). 
However, such static features can be removed from $z_{i,t}$ since the utility corresponding to these static features can be accounted in the constant term, i.e., 
 the intercept in $\alpha_{i_t}'(1, z_{i,t})$. In other words, if we only include static features of the products, the context-based pricing problem reduces to the one without any context.} }}

Under Assumption A, we have the following theoretical result on the regret
of the CSMP algorithm.

\begin{theorem}\label{thm:main}
Let input parameter $c\geq 20/l_1^2$; the expected
regret of algorithm CSMP is 
\begin{equation}\label{eq:detailed_regret}
R(T)=O\left(\frac{d^2 \log^2 (dT)}{\min_{i\in[n]}q_i^2 }+d\sqrt{mT}
\log{T}\right).
\end{equation}
In particular, if $q_i=\Theta(1/n)$ for all $i\in[n]$
and we hide the logarithmic terms, 
 then when $T\gg n$, the expected regret is at most
$\tilde{O}(d\sqrt{mT})$.
\end{theorem}

{ Here we briefly discuss the very high-level ideas of proving Theorem \ref{thm:main}, with the technical details deferred to the online supplement. One key step of proving the main part of the regret $\tilde O(d\sqrt{mT})$, as opposed to the typical regret $\tilde O(d\sqrt{nT})$ for single-product pricing without clustering, is that we are able to identify each neighborhood correctly. This is achieved by our technique of price perturbation which guarantees that our estimated parameter $\hat\theta_{i,t}$ converges to the true $\theta_i$ in a sufficiently fast rate. As a result, when $t\geq \bar t=\Theta({d^2 \log^2 (dT)}/({\min_{i\in[n]}q_i^2 }))$ (this is why we have regret $O({d^2 \log^2 (dT)}/({\min_{i\in[n]}q_i^2 }))$, which is incurred before $\bar t$), all neighborhoods are identified correctly with high probability. Conditioned on this, our problem is basically reduced to the pricing of $m$ ``products'' which gives us the regret $\tilde O(d\sqrt{mT})$. 

Although the key ideas are quite simple, the proofs are technical and involved which differ from the existing literature. For instance, compared with the bandit with clustering literature (see e.g., \citealt{cesa2013gang,gentile2014online,nguyen2014dynamic,gentile2016context}), our action set (prices) is continuous instead of finite and we have to exploit unique structure of revenue function (e.g., Assumption A) by using price perturbation techniques. Moreover, we do not assume stochasticity of context $x_{i,t}$ (e.g., contexts are i.i.d., which is also assumed in contextual pricing literature such as \citealt{chen2015statistical,qiang2016dynamic,nambiar2016dynamic,ban2017personalized} besides the bandit literature mentioned earlier) but only a weaker variation assumption (see Assumption A-3). This relaxation requires us to use a different matrix analysis technique instead of directly applying matrix concentration inequalities (see Lemma \ref{lemma:emp_fisher_info} in the online supplement). 

}

We have a number of remarks about the CSMP algorithm and the result on regret, following in order.

\medskip
\begin{remark}
\textbf{(Comparison with single-product pricing)} Our pricing policy achieves the regret $\tilde O(d\sqrt{mT})$. A question arises as to how it compares with the baseline single-product pricing algorithm that treats each product separately. \citet{ban2017personalized} consider a single-product pricing problem with demand covariates. 
According to Theorem 2 in \citet{ban2017personalized}, their algorithm, when applied to each product $i$ in our setting separately, achieves the regret $\tilde O(d\sqrt{T_{i,T}})$.  Therefore, adding together all products $i\in[n]$, the upper bound of the total regret is $\tilde O(d\sqrt{nT})$. When the number of clusters $m$ is much smaller than $n$, the regret $\tilde O(d\sqrt{mT})$ of CSMP significantly improves the total regret obtained by treating each product separately. 
\end{remark}

\medskip
\begin{remark}
\textbf{(Lower bound of regret)} 
To obtain a lower bound for the regret of our problem, we consider a special case of our model in which the decision maker knows the underlying true clusters $\mathcal{N}_j$. Since this is a special case of our problem (which is equivalent to single-product pricing for each cluster $\mathcal{N}_j$), the regret lower bound of this problem applies to ours as well. 
Theorem 1 in \citet{ban2017personalized} shows that the regret lower bound for each cluster $j$ has to be at least $\Omega\left(d\sqrt{\tilde{T}_{j,t}}\right)$. In the case that $\tilde q_j=1/m$ for all $j\in[m]$, it can be derived that the regret lower bound for all clusters has to be at least $\Omega(d\sqrt{mT})$. {This implies that the regret of the proposed CSMP policy is optimal up to a logarithmic factor.}
\end{remark}

\medskip
\begin{remark}
\textbf{(Improving the regret for large $n$)} When $n$ is large, the first term in our regret bound $O(d^2 \log^2 (dT)/\min_{i\in[n]}q_i^2 )$ will also become large. For instance, if $q_i=O(1/n)$ for all $i\in[n]$, then this term becomes $O(d^2n^2\log^2(dT))$. One way to improve the regret, although it requires prior knowledge of $\gamma$, is to conduct more price exploration during the early stages. Specifically, if the confidence bound $B_{i,t-1}$ of product $i$ is larger than $\gamma/4$, in Step 4, we let the price perturbation $\Delta_{i,t}$ be $\pm \Delta_0$ to introduce sufficient price variation (otherwise let $\Delta_t$ be the same as in the original algorithm CSMP). Following a similar argument as in Lemma \ref{lemma:emp_fisher_info} in the supplement, it roughly takes $O(d \log (dT)/\min_{i\in[n]}q_i )$  time periods before all $B_{i,t-1}<\gamma/4$, so the same proof used in
Theorem \ref{thm:main} appplies. Therefore, when $q_i=O(1/n)$ for all $i\in[n]$, the final regret upper bound is $O(dn\log (dT)+d\log T\sqrt{mT})$. 
\end{remark}

\medskip
{
\begin{remark}
\label{cluster}
\textbf{(Relaxing the cluster assumption)} Our theoretical development assumes that products within the same cluster have exactly the same parameters $\theta_i$. 
This assumption can be relaxed as follows. Without loss of generality, let us assume all products have different $\theta_i$. Define two products $i_1,i_2$ as in the same cluster if they satisfy $||\theta_{i_1}-\theta_{i_2}||_2\leq \gamma_0$ for some positive constant $\gamma_0$ with $\gamma_0<\gamma/3$. 
Our policy in Algorithm \ref{alg:CSMP} can adapt to this case by modifying Step 2 to 
\[
\hat{\mathcal{N}}_{i,t}=\{i'\in[n]:\; ||\hat{\theta}_{i',t-1}-\hat{\theta}_{i,t-1}||_2\leq 2{B}_{i',t-1}+2{B}_{i,t-1} \}.
\]
The reason we make this modification is that when $t$ is within certain range, $||\hat{\theta}_{i',t-1}-\hat{\theta}_{i,t-1}||_2> 2{B}_{i',t-1}+2{B}_{i,t-1}$ implies that, with high probability, $||{\theta}_{i',t-1}-{\theta}_{i,t-1}||_2> {B}_{i',t-1}+{B}_{i,t-1}>\gamma_0$. This shows we can correctly differentiate the products which are not in the same cluster.
As a result, under this modification, our algorithm CSMP has the following performance. 
If { $T\leq O(1/(\max_i q_i^2\gamma_0^4))$, the regret is at most $\tilde O(d\sqrt{mT}+\min\{\gamma_0^{2}\sum_j\tilde q_j^2 T^2,T\})$.} Thus when $T$ is small, the 
main difference with Theorem \ref{thm:main} is the extra term $O(\min\{\gamma_0^{2}\sum_j\tilde q_j^2 T^2,T\})$ due to relaxation of clusters, and if $\gamma_0$ is small, we still have the overall regret better than $\tilde O(d\sqrt{nT})$ without any clustering. On the other hand, for large $T$ and in particular, when $T\to\infty$, we show that the regret will approach $\tilde O(d\sqrt{nT})$. Intuitively, this is because when the data is no longer scarce, our clustering actually identifies each product $i$ as its own cluster, reducing to the single-product pricing algorithm without clustering. For detailed analysis of this relaxation, we refer the interested readers to Section \ref{ap_sec:relax_cluster} in the online supplement. 



One relevant stream of literature to this setting is the so-called bandit with model mis-specification, which assumes that the reward function is mis-specified with error $\varepsilon$ (see e.g., \citealt{crammer2013multiclass,ghosh2017misspecified,lattimore2020learning,foster2020beyond,foster2020adapting}), and they show that the part of regret related to $\varepsilon$ has to be $\Omega(\varepsilon T)$. Our method in this setting is different in that we only take advantage of the $\gamma_0$-different parameters ($\gamma_0$ is typically very small) in the same cluster when data is scarce (i.e., $T$ is small). As more data is gathered, the algorithm naturally converges to single-product pricing, making our regret in the long-run still being sublinear in $T$.

\end{remark}
}

\section{Simulation Results with Synthetic and Real Data}\label{sec:simulation}
{ This section provides the simulation experiment results for algorithm CSMP. First, we conduct a simulation study using synthetic data in Section \ref{subsec:synthetic_simulation} to illustrate the effectiveness and robustness of our algorithms against several benchmark approaches. Second, the simulation results using a real dataset from Alibaba are provided in Section \ref{subsec:real_simulation}. 
Finally, we summarize all numerical experiment results in Section \ref{subsec:data_insight}.}

{
\subsection{Simulation using synthetic data}\label{subsec:synthetic_simulation}
{In this section, we demonstrate the effectiveness of our algorithms using some synthetic data simulation. We first show the performance of CSMP against several benchmark algorithms. Then, several robustness tests are conducted for CSMP. 
The first test is for the case when clustering assumption is violated (i.e., parameters within the same cluster are slightly different).
The second test is when the
demand covariates $z_{i,t}$ contain some features that change slowly in a deterministic manner.
Finally, we test CSMP with a misspecified demand model.
 }

We shall compare the performance of our algorithms with the following benchmarks:
\begin{itemize}
\item The Semi-Myopic Pricing (SMP) algorithm, which treats each product independently (IND), and we refer to it as SMP-IND.

\item The Semi-Myopic Pricing (SMP) algorithm, which treats all products as one (ONE) single cluster, and we refer to the algorithm as SMP-ONE.

\item The Clustered Semi-Myopic Pricing with $K$-means Clustering (CSMP-KMeans), which uses $K$-means clustering for product clustering in Step 2 of CSMP.
\end{itemize}
The first two benchmarks are natural special cases of our algorithm. Algorithm SMP-IND skips the clustering step in our algorithm and always sets the neighborhood as $\hat{\mathcal{N}}_t=\{i_t\}$; while SMP-ONE keeps $\hat{\mathcal{N}}_t=\mathcal{N}$ for all $t\in [T]$. The last benchmark is to test the effectiveness of other classical clustering approach for our setting, in which we choose $K$-means clustering as an illustrative example because of its popularity.

\subsubsection{Logistic demand with clusters } For illustration of a GLM demand, we simulate using a logistic function. We set the time horizon $T=30,000$, the searching probability $q_i=1/n$ for all $i\in [n]$ where $n=100$, and the price range $\underline p =0$ and $\overline p=10$. In this study, it is assumed that all $n=100$ products have $m=10$ clusters (with products randomly assigned to clusters). Within a cluster $j$, each entry in $\alpha_j$ is generated uniformly from $[-L/\sqrt{d+2},L/\sqrt{d+2}]$ with $L=10$, and $\beta_j$ is generated uniformly from $[-L/\sqrt{d+2},0)$ (to guarantee that $||\theta_i||_2\leq L$). For demand covariates,
each feature in $z_{i,t}$, with dimension $d=5$, is generated independently and uniformly from $[-1/\sqrt{d},1/\sqrt{d}]$ (to guarantee that  $||z_{i,t}||_2\leq 1$). For the parameters in the algorithms, we let $\Delta_0 = 1$; and for the confidence bound $B_{i,t}=\sqrt{c(d+2)\log(1+t)/\lambda_{\min}(V_{i,t})}$, we first let $c=0.8$ and then test other values of $c$ for sensitivity analysis. For the benchmark CSMP-KMeans, we need to specify the number of clusters $K$; since the true number of clusters $m$ is not known \textit{a priori}, we test different values of $K$ in $\{5,10,20,30\}$. Note that when $K=10$, the performance of CSMP-KMeans can be considered as an oracle since it correctly specifies the true number of product clusters.

To evaluate the performance of algorithms,  we adopt both the cumulative regret in \eqref{eq:def_regret} and the percentage revenue loss defined by
\begin{equation}\label{eq:def_perc_rev_loss}
    L^{\pi}(T)=\frac{R^{\pi}(T)}{\sum_{t=1}^{T}\mathbb{E}[r_t(p_t^*)]},
\end{equation}
which  measures the percentage of revenue loss with respect to the optimal revenue. Obviously, the percentage revenue loss and cumulative regret are equivalent, and a better policy leads to a smaller regret and a smaller percentage revenue loss.

For each experiment, we conduct 30 independent runs and take their average as the output. 
We also output the standard deviation of percentage revenue loss for all policies in Table \ref{tab:std}. It can be seen that our policy CSMP has quite small standard deviation, so we will neglect standard deviation results in other experiments.

{
We recognize that a more appropriate measure for evaluating an algorithm is the regret (and percentage of loss) of expected total profit (instead of expected total revenue). We choose the latter for the following reasons. First, it is consistent with the objective of this paper, which is the choice of the existing literature. Second, it is revenue, not profit, that is being evaluated at our industry partner, Alibaba. Third, even if we wish to measure it using profit, the cost data of products are not available to us, since the true costs depend on such critical things as terms of contracts with suppliers, that are confidential information.
}

\begin{figure}[!h]
\begin{subfigure}[t]{0.5\textwidth}
\centering
\includegraphics[width=1.1\textwidth]{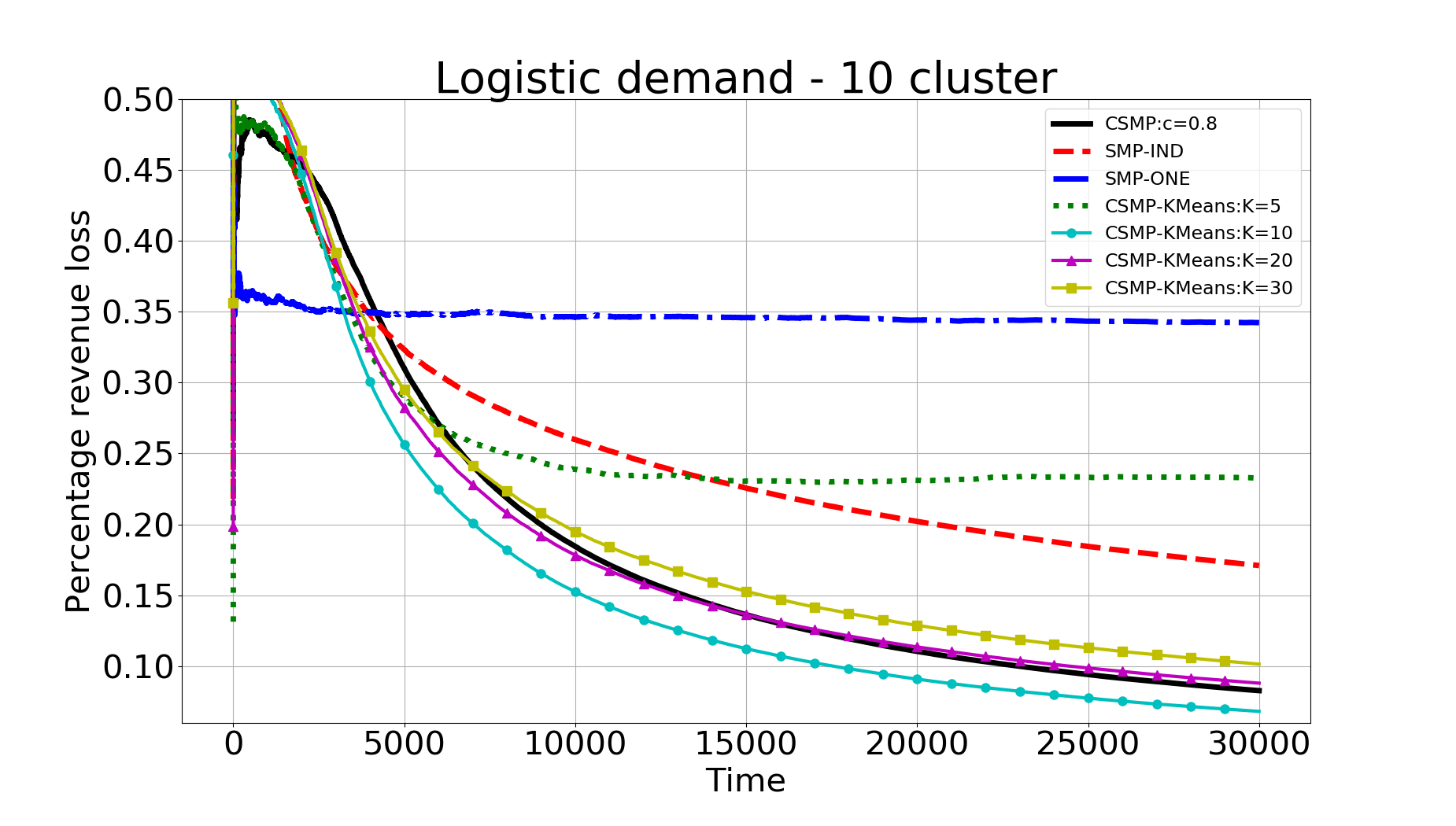}
\caption{Plot of percentage revenue loss}
\end{subfigure}
~
\begin{subfigure}[t]{0.5\textwidth}
\centering
\includegraphics[width=1.1\textwidth]{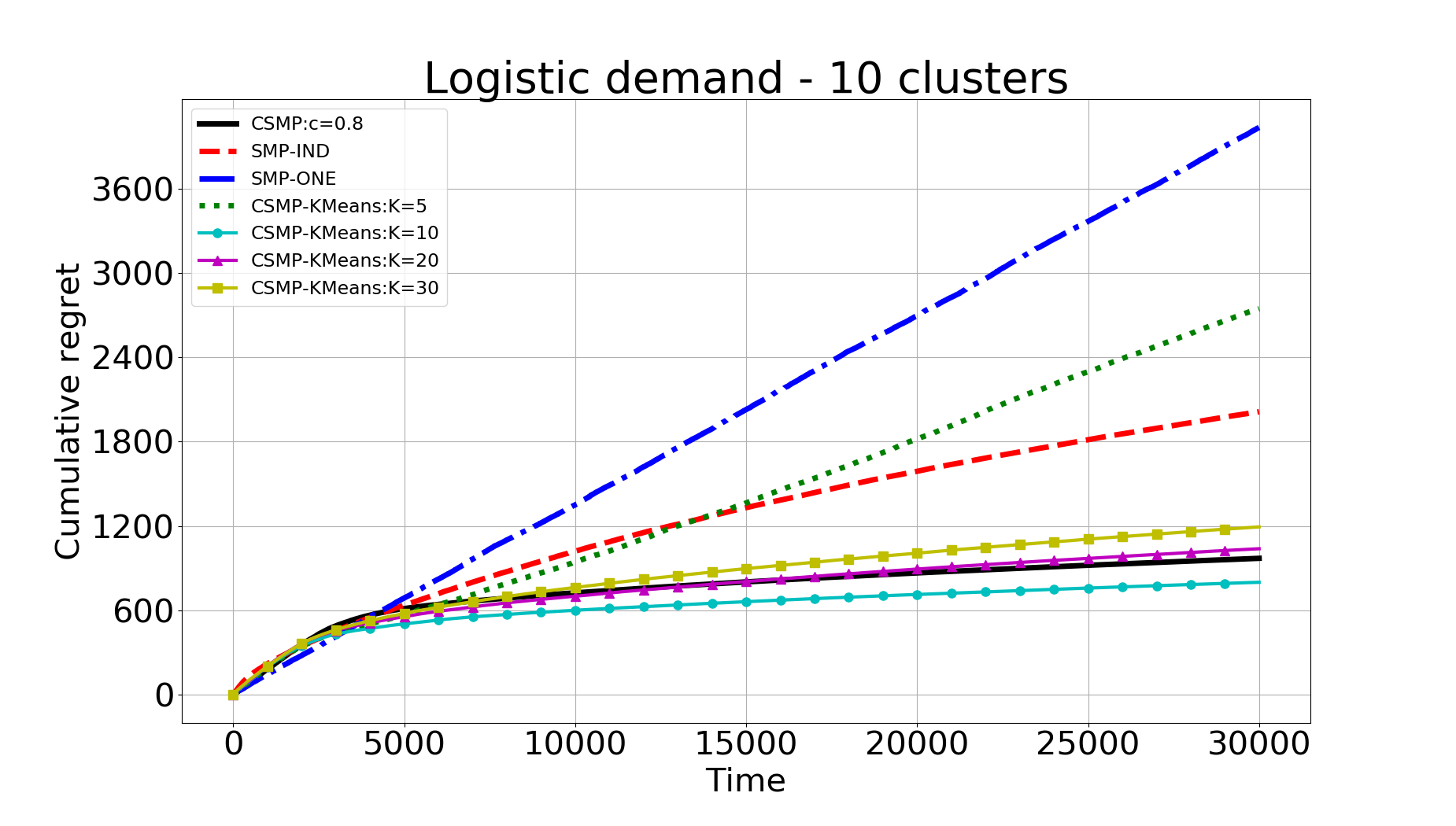}
\caption{Plot of cumulative regret}
\end{subfigure}
\vspace{.1in}
\caption{Performance of different policies for logistic demand with 10 clusters.   The graph on the left-hand side shows the percentage revenue loss of all algorithms, and the graph on the right-hand side shows the cumulative regrets for each algorithm. 
}
\label{fig:logit_10_cluster}
\end{figure}

\begin{table}[!t]
\centering
\begin{tabular}{l|llllll}
\hline
\hline
        & $t=5,000$ & $t=10,000$ & $t=15,000$ & $t=20,000$ & $t=25,000$ & $t=30,000$ \\ \hline
CSMP    & 1.83    & 0.97      & 0.70      & 0.57      & 0.47      & 0.40  \\ 
SMP-IND & 1.32     & 0.88      & 0.92      & 0.81      & 0.78      & 0.73  \\
SMP-ONE & 2.34     & 2.15      & 1.75      & 1.44      & 1.46      & 1.44  \\
CSMP-KMeans:$K=5$ & 2.08     & 1.97      & 1.95     & 2.26      & 2.22      & 2.19  \\
CSMP-KMeans:$K=10$ & 2.06     & 1.53      & 1.09      & 0.87      & 0.74      & 0.66  \\
CSMP-KMeans:$K=20$ & 2.12     & 1.36      & 1.15      & 1.02      & 0.91      & 0.82  \\
CSMP-KMeans:$K=30$ & 1.41     & 0.88      & 0.77      & 0.67      & 0.59      & 0.49  \\

\hline
\end{tabular}
\vspace{.1in}

\caption{Standard deviation ($\%$) of percentage revenue loss corresponding to different time periods for logistic demand with 10 clusters.}
\label{tab:std}
\end{table}

\begin{table}[!t]
\centering
\begin{tabular}{l|llllll}
\hline
\hline
                   & $c=0.5$ & $c=0.6$ & $c=0.7$ & $c=0.8$ & $c=0.9$ & $c=1.0$ \\
\hline
Mean               & 8.56    & 8.28    & 8.52    & 8.27    & 8.56    & 8.72    \\
Standard deviation & 0.73    & 0.51    & 0.73    & 0.40    & 0.66    & 0.35   \\
\hline
\end{tabular}
\vspace{.1in}

\caption{Mean and standard deviation ($\%$) of percentage revenue loss of CSMP (logistic demand with 10 clusters) with different parameters $c$.}
\label{tab:sensitivity}
\end{table}

The results are shown in Figure \ref{fig:logit_10_cluster}. According to this figure, our algorithm CSMP outperforms all the benchmarks except for CSMP-KMeans when $K=m=10$.
CSMP-KMeans with $K=10$ has the best performance, which is not surprising because it uses the exact and correct number of clusters. However, in reality the true cluster number $m$ is not known. We also test CSMP-KMeans with $K=5,20,30$. We find that when $K=20$, its performance is similar to (slightly worse than) our algorithm CSMP. When $K=5,30$, the performance of CSMP-KMeans becomes much worse (especially when $K=5$). For the other two benchmarks SMP-ONE and SMP-IND, their performances are not satisfactory either, with SMP-ONE has the worst performance because clustering all products together leads to significant error. Sensivitiy results of CSMP with different parameters $c$ are presented in Table \ref{tab:sensitivity}, and it can be seen that CSMP is quite robust with different values of $c$.
}

{\subsubsection{Logistic demand with relaxed clusters}\label{subsubsec:relax_cluster}
As we discussed in Section \ref{subsec:theory}, strict clustering assumption might not hold and sometimes products within the same cluster are slightly different. This experiment tests the robustness of CSMP when parameters of products in the same cluster are slightly different. To this end, after we generate the $m=10$ centers of parameters (with each center represented by $\theta_j$), for each product $i$ in the cluster $j$, we let $\theta_i=\theta_j+\Delta\theta_i$ where $\Delta\theta_i$ is a random vector such that each entry is uniformly drawn from $[-L/(10\sqrt{d+2}),L/(10\sqrt{d+2})]$. All the other parameters are the same as in the case with $10$ clusters. 
Results are summarized in Figure \ref{fig:logit_no_cluster}, and it can be seen that the performances of all algorithms are quite similar as in Figure \ref{fig:logit_10_cluster}. }

\begin{figure}[!h]
\begin{subfigure}[t]{0.5\textwidth}
\centering
\includegraphics[width=1.1\textwidth]{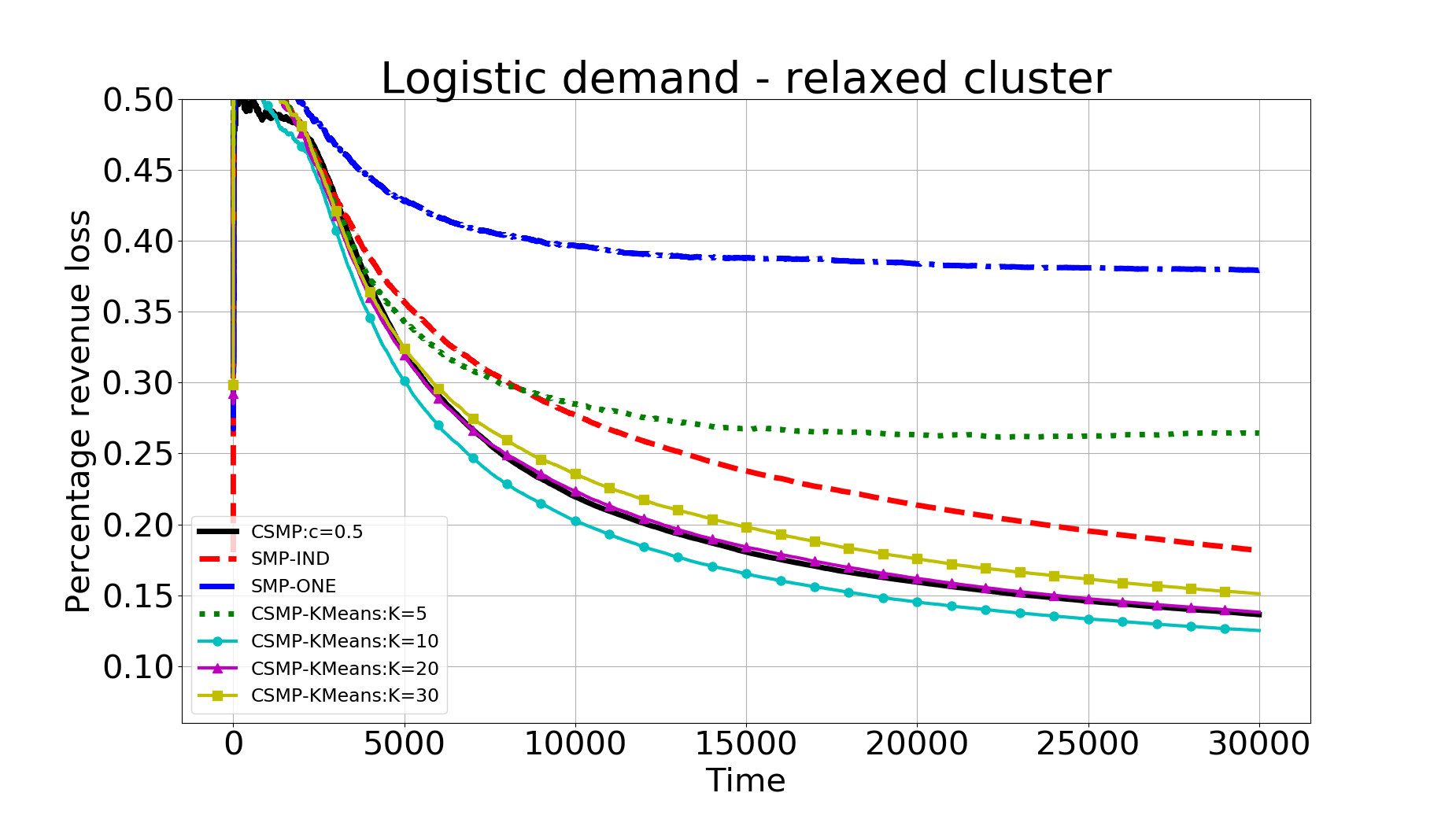}
\caption{Plot of percentage revenue loss}
\end{subfigure}
~
\begin{subfigure}[t]{0.5\textwidth}
\centering
\includegraphics[width=1.1\textwidth]{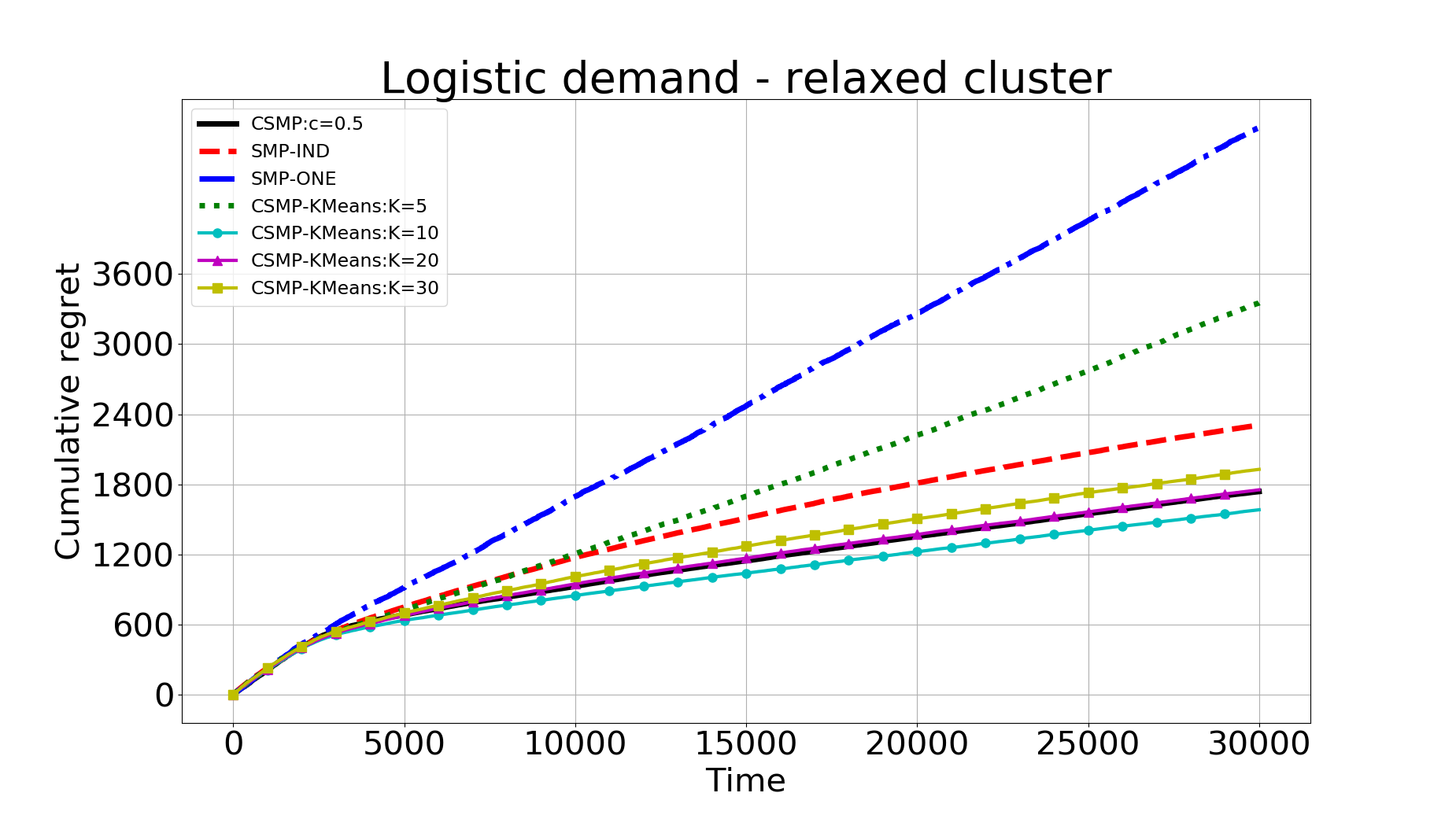}
\caption{Plot of cumulative regret}
\end{subfigure}

\vspace{.1in}

\caption{Performance of different policies for logistic demand with relaxed clusters. }
\label{fig:logit_no_cluster}
\end{figure}

{
\subsubsection{Logistic demand with almost static features}
As we discussed after Assumption A-3, in some applications there might be features that have little variations (nearly static). We next test the robustness of our algorithm CSMP when the feature variations are small. 
 To this end, we assume that one feature in $z_{i,t}\in\mathbb{R}^d$ for each $i\in [n]$ is almost static. More specifically, we let this feature be constantly $1/\sqrt{d}$ 
 for 100 periods, then change to $-1/\sqrt{d}$ 
 for another 100 periods, then switch back to $1/\sqrt{d}$ 
 after 100 periods, and this process continues. 
 The numerical results against benchmarks are summarized in Figure \ref{fig:logit_10_cluster_as}. It can be seen that with such an almost static feature, the performances of algorithms with clustering become worse, but they still outperform the benchmark algorithms. 
In particular, CSMP (with parameter $c=0.1$ after a few trials of tuning) still has promising performance, showing its robustness with small feature variations of some products. 

\begin{figure}[!h]
\begin{subfigure}[t]{0.5\textwidth}
\centering
\includegraphics[width=1.1\textwidth]{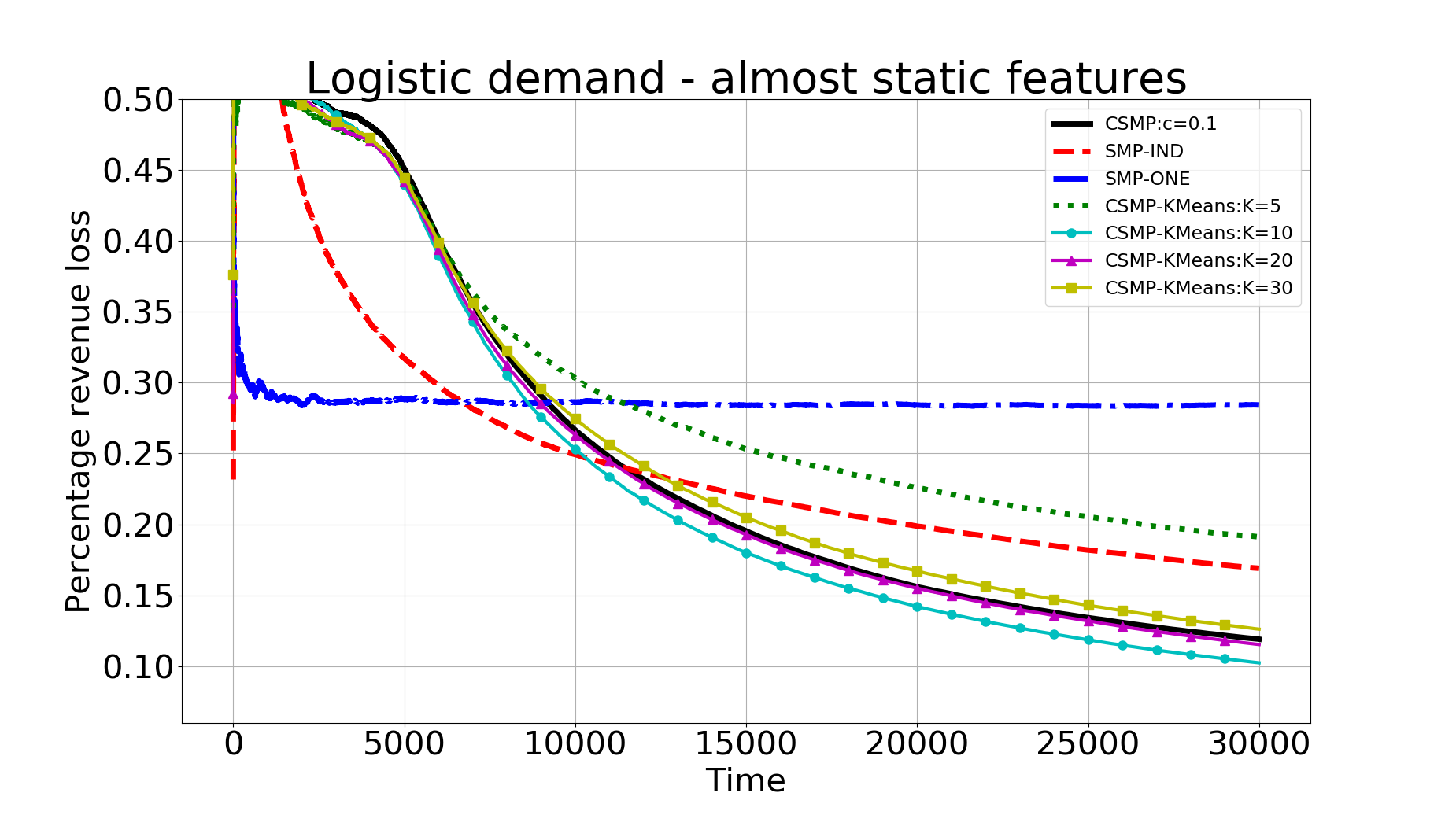}
\caption{Plot of percentage revenue loss}
\end{subfigure}
~
\begin{subfigure}[t]{0.5\textwidth}
\centering
\includegraphics[width=1.1\textwidth]{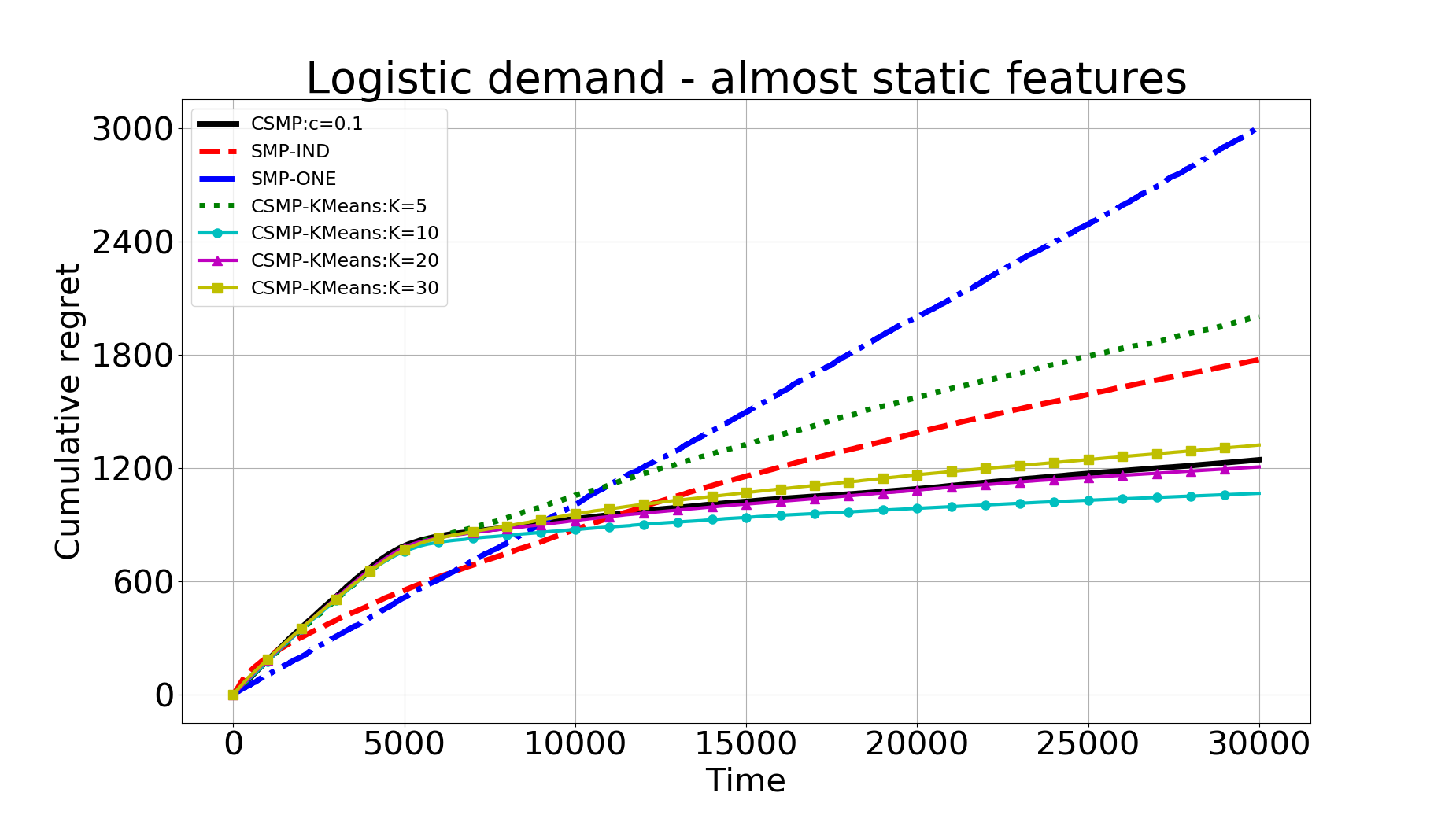}
\caption{Plot of cumulative regret}
\end{subfigure}
\caption{Performance of different policies for logistic demand with 10 clusters and almost static features. }
\label{fig:logit_10_cluster_as}
\end{figure}
}

{
\subsubsection{Logistic demand with model misspecification}
In real applications, it may happen that the demand model is misspecified. In this experiment, we consider a misspecified logistic demand model.
Specifically, we let the expected demand of product $i$ be
$
{1}/({1+\exp(f_i(z_{t},p_t))}), 
$
where the utility function 
\[
f_i(z_t,p_t):=c_{i,0}+\sum_{k=1}^{d}c_{1,i,k}z_{t,k}+\sum_{k=1}^{d}c_{2,i,k}z_{t,k}^2+\sum_{k=1}^{d}c_{3,i,k}z_{t,k}^3+\beta_{1,i}p_t+\beta_{2,i}^2p_t^2+\beta_{3,i}p_t^3
\]
is a third degree polynomial of $z_t,p_t$, where $c_{i},\beta_i$ are unknown parameters, and $z_{t,k}$ represents te $k$-th component of $z_t$. To generate this misspecified demand model, we let $c_{l,i,k}\in [-L/\sqrt{3(d+2)},L/\sqrt{3(d+2)}]$ with $l\in\{1,2,3\},k\in [d]$, $c_{i,0}\in [-L/\sqrt{d+2},L/\sqrt{d+2}]$, and $\beta_{l,i}\in [-L/\sqrt{3(d+2)},0)$ with 
$l\in \{1,2,3\}$, 
be all drawn uniformly. All the other input parameters for the problem instance are the same as in the case of logistic demand with 10 clusters. 

To test the robustness of the misspecified CSMP, it is compared with CSMP which correctly specifies the demand model. We call the benchmark the CSMP-Oracle. The numerical results are summarized in  Figure \ref{fig:logit_misspecification}. As seen, when compared with the oracle, the misspecified CSMP has slightly worse performance as expected. But the overall difference in percentage revenue loss is only $3.48\%$, showing that our algorithm CSMP is rather robust with such a model misspecification.


\begin{figure}[h!]
\begin{subfigure}[t]{0.48\textwidth}
\centering
\includegraphics[width=1.1\textwidth]{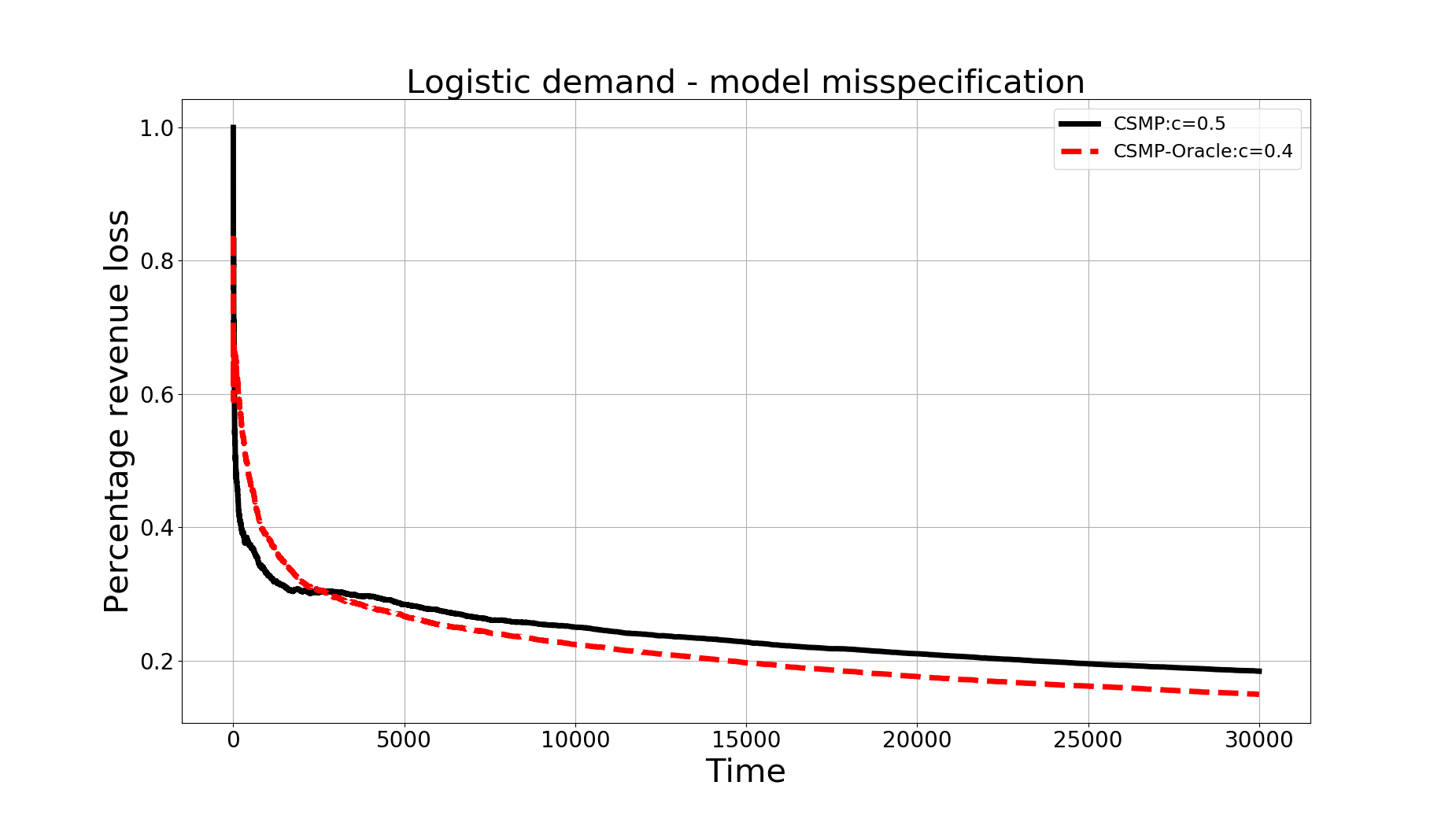}
\caption{Plot of percentage revenue loss}
\end{subfigure}
~
\begin{subfigure}[t]{0.48\textwidth}
\centering
\includegraphics[width=1.1\textwidth]{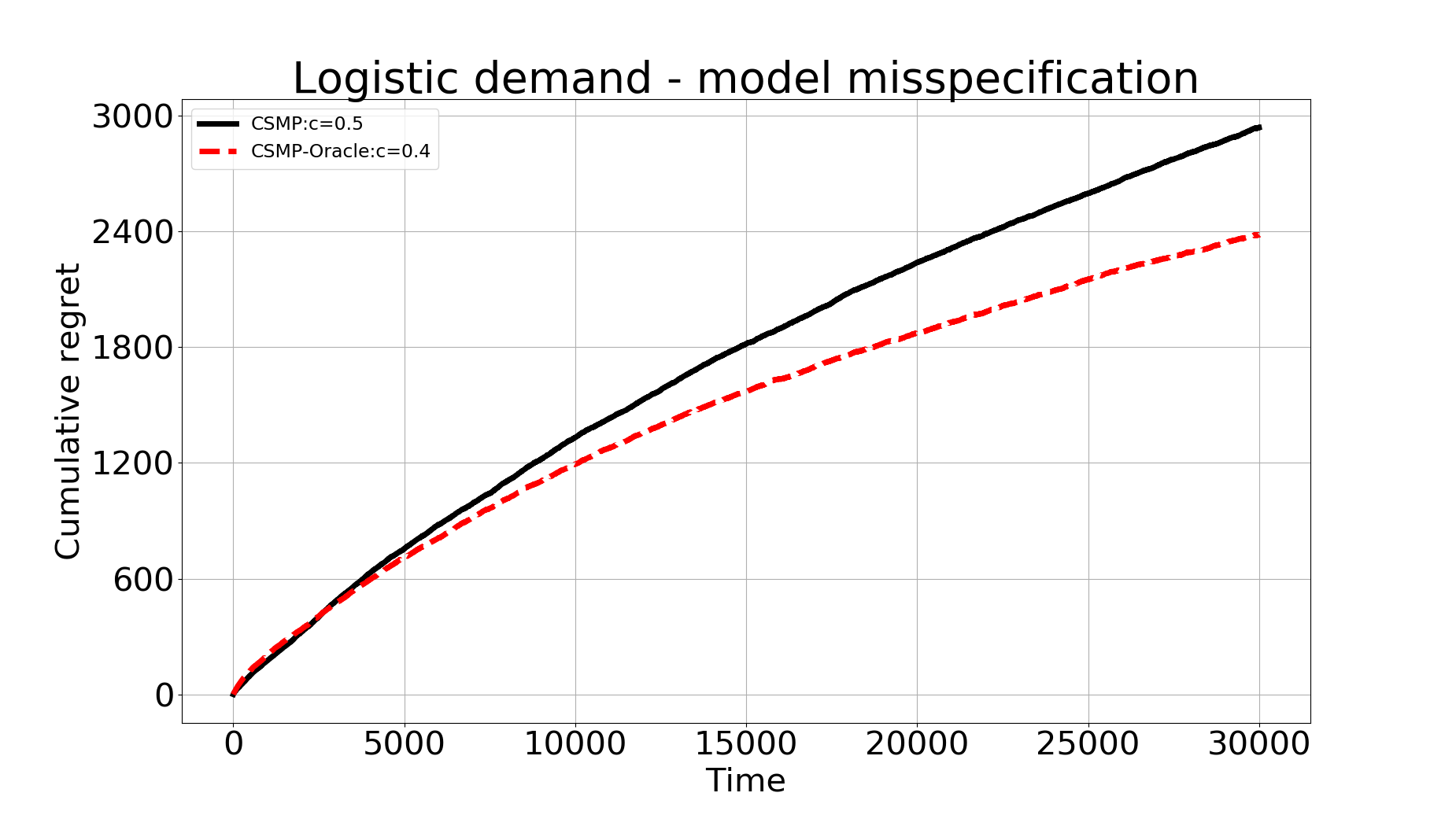}
\caption{Plot of cumulative regret}
\end{subfigure}
\caption{Performance of CSMP with (misspecified) logistic demand versus the oracle.}
\label{fig:logit_misspecification}
\end{figure}

}

{
\subsection{Simulation using real data from Alibaba}\label{subsec:real_simulation}
This section presents the results of our algorithms (for illustration, we use CSMP with logistic demand) and other benchmarks using a real dataset provided by Alibaba. To better simulate the real demand process, we fit the demand data to create a  sophisticated ground truth model (hence our algorithm CSMP may have a model misspecification). 
Before presenting the results, we introduce the dataset and pre-processing of the data.

\textbf{The dataset.} To motivate our study of pricing for low-sale products, we extract sales data 
from 05/29/2018 to 07/28/2018. During this period, nearly 75,000 products were offered by Tmall Supermarket. There are more than 21.6\% (i.e., 16,000) products with average numbers of daily unique visits less than 10. Among all these low-sale products, Alibaba provided us with a test dataset comprising $100$ products that have at least one sale during the 61-day period, and at least two prices charged with each price offered to more than $10\%$ of all customers. Because these selected products have relatively sufficient variation of prices and different observations of customers' purchases, demand parameters can be estimated quite accurately using the sales data in the dataset. {As the focus of this paper is on single-product pricing where the demand of each product depends on its own price, this numerical study using the real dataset is also conducted under this setting. In Section \ref{app_sec:price_dependency} in the online supplement, we conduct a data analysis to show that this assumption is indeed reasonable for this dataset.}

For the features of products, we are provided by Alibaba with 5 features (hence $d=5$), that are described below: 
\begin{itemize}
    \item Average gross merchandise volume (GMV, i.e., product revenue) in past 30 days.
    \item Average demand in past 30 days.
    \item Average number of unique buyers (UB, i.e., unique IP which makes the purchase) in past 30 days.
    \item Average number of unique visitors (UV) in past 30 days.
    \item Average number of independent product views (IPV, i.e., total number of views on the product, including repetitive views from the same user) in past 30 days.
\end{itemize}

These features are selected by Alibaba's feature engineering team\footnote{We requested to include some other features, 
such as number/score of customer ratings and competitor's price on similar product, but were unable to obtain such data due to technical reasons.} 
 (via a recursive feature elimination approach from a raw set of features).
Note that these features are not exogeneous, since features in the future can be affected by current pricing decision. Such endogenous features are often used in the demand forecasting literature. For instance, a time series model uses past demand to predict future demand (see e.g., \citealt{brown1959statistical}); an artificial neural network (ANN) model uses historical demand data of composite products as features for demand prediction \citep{chang2005evolving}. In the pricing literature, some endogenous features have also been used. For example, in  \cite{ban2017personalized,bastani2019meta}, their model features include  auto loan data, e.g., competitors' rate, that are affected by the rate offered by the decision maker (the auto loan company). Incorporating the impact of pricing decisions on features leads to challenging 
dynamic programming problem with partial information. 
Hence, features are considered as given and we only optimize for current period (i.e., ignoring the long-run effect of the current pricing
 decision).
 
{To understand the data and features better, we provide a summary statistics on the data and feature. First, we calculated the number of prices charged during the time horizon, average demand per day, average UV per day, and average IPV per day, for each one of the selected 100 products. Table \ref{tab:summary} summarizes the mean and standard deviation of each of the data.
\begin{table}[!t]
\centering
\begin{tabular}{l|llll}
\hline
\hline
                   & Number of prices & Demand    & UV           & IPV           \\ 
\hline
Mean                & 7.13  & 0.71 & 6.88& 10.13    \\
Standard deviation & 5.16                             & 0.72                        & 2.04                       & 3.38    \\
\hline
\end{tabular}
\vspace{.1in}
\caption{Summary statistics of average data of the 100 products.}
\label{tab:summary}
\end{table}
Then, to understand the variation of features, we calculated the standard deviation of each of the five features of every product. In Table \ref{tab:summary_feature} we summarize the mean and standard deviation of feature variations of all products. 
\begin{table}[!t]
\centering
\begin{tabular}{l|lllll}
\hline
\hline
                   & Feature GMV & Feature demand    & Feature UB           & Feature UV & Feature IPV           \\ 
\hline
Mean                & 12.93  &	0.34 &	0.49 &	5.04 &	7.73   \\
Standard deviation & 17.97 &	0.60 &	0.89 &	9.21 &	13.95   \\
\hline
\end{tabular}
\vspace{.1in}
\caption{Summary statistics of feature variation of the 100 products.}
\label{tab:summary_feature}
\end{table}
} 
 
To run simulation using the real dataset, we first create a ground truth model for the demand. We consider two ground truth models in this simulation study. The first one is the commonly used logistic demand function (hence no model misspecification for our algorithm CSMP), and the second is a random forest model (as used in simulation study of \citealt{nambiar2016dynamic}, hence there is model misspecification for CSMP). We use the demand data of each product to fit these two demand models, and then apply them to simulate the demand process. To test the accuracy of these two demand models, we calculated the average ROC-AUC score (among the selected products) of the logistic and random forest model respectively. The results show that the ROC-AUC score of the random forest model is 11.7\% higher than that of the logistic model (thus the random forest model fits the reality better). 

We want to generate customer's arrival at each time $t$, i.e., the product $i_t$ a customer chooses to search. Since the dataset contains the daily number of unique visitors for each product $i$, the arrival process $i_t$ is simulated by randomly permuting the unique visitors of each product on each day. For instance, if on day $1$, product $1$ and product $2$ have $2$ and $3$ unique visitors respectively; then $i_t$ for $t=1,\ldots,5$ can be $1,2,2,1,2$, which is a random permutation of the unique visitors for product $1$ and $2$.

\textbf{Numerical results for the algorithms. } 
We first provide the specifications of the parameters in the CSMP algorithm in Algorithm \ref{alg:CSMP}.
\begin{itemize}
\item The confidence bound $B_{i,t}$ is  $\sqrt{c(d+2)\log (1+t)/\lambda_{\min}(V_{i,t})}$, where $c=0.01$ for logistic demand and $c=0.05$ for random forest demand (selected by a few trials of different values). 

\item The price lower bound of each product is $50\%$ lower than its lowest price during the 61-day period, and the price upper bound is $50\%$ higher than its highest price during this period of time. 

\item The basic price perturbation parameter $\Delta_0$ of each product is set as the length of price range divided by $4$, i.e., $\Delta_0=(\overline{p}-\underline{p})/4$.
\end{itemize}

{For benchmark algorithms, they are the same as those in the previous subsection, with CSMP-KMeans have $K\in\{5,10,20,30\}$. In addition, we test another benchmark  proposed  in \citet{keskin2016chasing}. More specifically, this benchmark assumes a simple linear demand model as $\mathbb{E}[d_{i,t}]=\alpha_{i,t}+\beta_{i,t}'p_{i,t}$ with changing parameters $\alpha_{i,t},\beta_{i,t}$ but without demand covariates. Since this single-product pricing algorithm can be considered as a modified version of semi-myopic pricing, we call it semi-myopic pricing (SMP) with changing parameters (CP), or SMP-CP for short. In particular, among the algorithms proposed in \citet{keskin2016chasing}, we choose the ``Moving Window Policy'' as it has the best empirical performance and choose the input parameter $\kappa=0.5$ as in the numerical experiments in Section 6.3 in \citet{keskin2016chasing}.
We plot the results of cumulative revenue at different dates in Figure \ref{fig:real_data}. }

It can be seen that all the methods using clustering have better performance,  and their performances are comparable.  It is interesting to note that for clustering using $K$-means method, their performances with different value of $K$ are actually quite close. Finally, it is observed that the advantage of using clustering with random forest model (i.e., misspecified model) is more than that with logistic model.

\begin{figure}[!h]
\begin{subfigure}[t]{0.48\textwidth}
\centering
\includegraphics[width=1\textwidth]{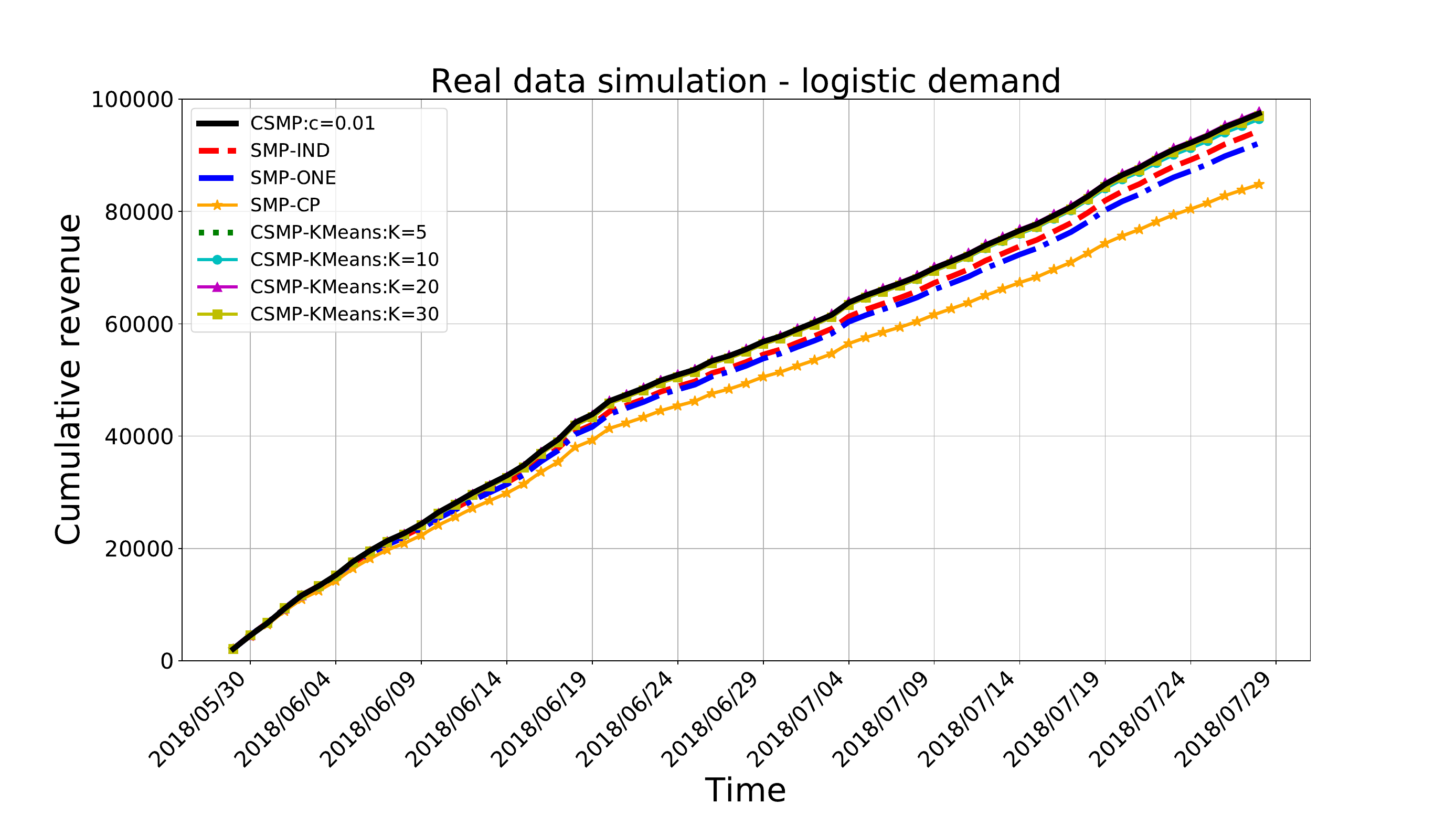}
\caption{Logistic demand model (without model misspecification)}
\end{subfigure}
~
\begin{subfigure}[t]{0.48\textwidth}
\centering
\includegraphics[width=1\textwidth]{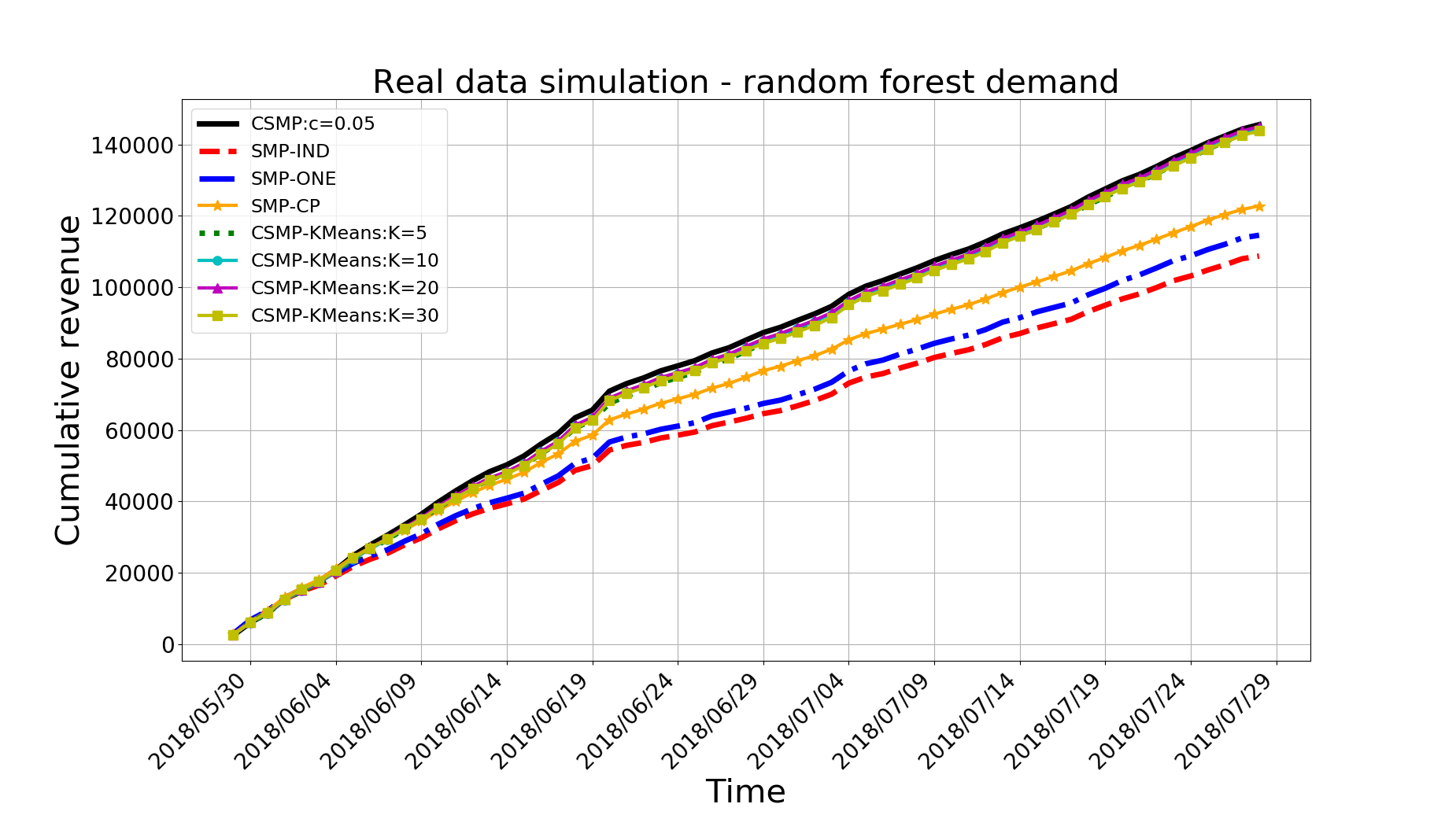}
\caption{Random forest demand model (with model misspecification)}
\end{subfigure}
\centering
\caption{Plot of cumulative revenue over different dates for two demand models}
\label{fig:real_data}
\end{figure}



}

{
\subsection{Summary of numerical experiments}\label{subsec:data_insight}
In this section we first 
present the simulation results using synthetic data under various scenarios to test the effectiveness and 
robustness of our algorithms, then we present the simulation 
results with real data from Alibaba using a more sophisticated ground truth demand model (for a more realistic simulation and robustness test under model misspecification). 
The main findings from the numerical study are 
summarized as follows.  
\begin{itemize}
\item In all the numerical results, pricing with clustering (either using our method in CSMP or classical $K$-means clustering with appropriate choice of $K$) outperforms the benchmarks of applying single-product pricing algorithm on each product or naively putting all products into a single cluster.

\item Dynamic pricing with $K$-means clustering method sometimes works as effectively as (and at times even better than) our algorithm CSMP. But its performance depends on the choice of the number of clusters $K$, which is unknown to the decision maker.

\item The CSMP algorithm is quite robust under different scenarios: slightly different demand parameters within the same cluster, 
near static or slowly changing features, and misspecified
ground truth demand model.

\end{itemize}
}

\section{Conclusion}\label{sec:conclusion}
With the rapid development of e-commerce, data-driven dynamic pricing is becoming increasingly important due to  the dynamic market environment and easy access to online sales data. While there is abundant literature on dynamic pricing of normal products, the pricing of products with low sales received little attention. The data from Alibaba Group shows that the number of such low-sale products is large, and that even though the demand for each low-sale product is small, the total revenue 
for all the low-sale products is quite significant. 
In this paper, we present data clustering and dynamic pricing algorithms to address this challenging problem. We believe that this paper is the first to integrate online clustering learning in dynamic 
pricing of low-sale products.

A learning algorithm is proposed in this paper which learns the demand and decides product clustering simultaneously on the fly.
We have established the regret bound for this under mild technical conditions. 
Moreover, we test our algorithms on a real dataset from Alibaba Group by simulating the demand function. Numerical results show that our algorithm outperforms the benchmarks, where one either considers all products separately, or treats all products as a single cluster. 

There are several possible future research directions.
The first one is an in-depth study of the method for product clustering (e.g., \citealt{gentile2014online,nguyen2014dynamic}). 
Second, 
to highlight the benefit of  clustering techniques for low-sale products, in this paper 
we study a dynamic pricing problem with sufficient inventory. 
One extension is to apply the clustering method for the revenue management problem with inventory constraint. 
Third, we believe that it will be interesting to include substitutability/complementarity of products and even assortment decisions. 

\section*{Acknowledgement}
We thank the Department Editor, Senior Editor, and anonymous referees for their detailed and constructive comments that have helped to considerably improve the exposition of this paper. 


\bibliographystyle{informs2014}
\bibliography{reference}

\ECSwitch


\ECHead{E-Companion}

In this Appendix, we present all the mising proofs in the mainbody of the paper. We also
prove the result discussed in Remark \ref{cluster} of Section \ref{sec:policy_and_results}
for a more general definition of clusters. 

\section{Proof of Theorem \ref{thm:main}}\label{ap_sec:proof_main}
First of all, we define $\tilde{q}_j:=\sum_{i\in \mathcal{N}_j}q_i$ as the probability that a customer views a product from cluster $j$.
Then, define the events
\[
\begin{split}
\mathcal{E}_{N,t}:=&\{\hat{\mathcal{N}}_t=\mathcal{N}_{i_t} \},\\
\mathcal{E}_{B_j,t}:=&\{||\tilde{\theta}_{j,t}-\theta_{j}||_2\leq \tilde B_{j,t} \},\\
\mathcal{E}_{V,t}:=&\left\{\lambda_{\min}\left(\sum_{s\in\tilde{\mathcal{T}}_{j_t,t}}u_su_s'\right)\geq \frac{\lambda_1\Delta_0^2 \sqrt{\tilde{q}_{j_t} t }}{8} \right\},\\
\end{split}
\]
where $\lambda_1=\min(1,\lambda_0)/(1+\overline{p}^2)$ and $\tilde{\theta}_{j,t}$ is the estimated parameters using data from $\tilde{\mathcal T}_{j,t}$, and 
\[
\tilde B_{j,t}=:\frac{\sqrt{c(d+2)\log(1+t))}}{\sqrt{\lambda_{\min}(\tilde{V}_{j,t})}}
\]
for some constant $c\geq 20/l_1^2$
and $\tilde{V}_{j,t}=I+\sum_{s\in\tilde{\mathcal{T}}_{j,t}} u_su_s'$.  These events hold at least with the following probabilities
\[
\begin{split}
\mathbb{P}(\mathcal{E}_{N,t})\geq & 1-\frac{2n}{t^2}\qquad\text{for $t>\bar t$},\\
\mathbb{P}(\mathcal{E}_{B_j,t})\geq & 1-\frac{1}{t}\qquad\text{for any $j\in[m],t\in\mathcal{T}$},\\
\mathbb{P}(\mathcal{E}_{V,t})\geq & 1-\frac{7n}{t}\qquad\text{for $t>2\bar t$},\\
\end{split}
\]
where $\bar t$ is defined in (\ref{eq:t_2_min}). The first inequality is from
our analysis after Lemma \ref{lemma:hat_N_eq_N}; the second inequality is from Corollary \ref{lemma:cluster_est_error}; the third inequality is from Lemma \ref{lemma:cluster_emp_fisher_info}. We further define $\mathcal{E}_{B,t}=\bigcup_{j\in[m]}\mathcal{E}_{B_j,t}$, then it holds with probability at least $1-m/t$ for any $t\in\mathcal{T}$. Now we define the event $\mathcal{E}_t$ as the union of $\mathcal{E}_{N,t}$, $\mathcal{E}_{B,t}$, and $\mathcal{E}_{V,t}$. This event holds with probability at least $1-10n/t$ obviously according to the probability of each event.

We split the regret by considering $t\leq 2\bar t$ and $t>2\bar t$, i.e.,
\[
\begin{split}
\sum_{t=1}^{T}\mathbb{E}[r_t(p_t^*)-r_t(p_t)]=  \sum_{t\leq 2\bar t}\mathbb{E}[r_t(p_t^*)-r_t(p_t)]+\sum_{t> 2\bar t}\mathbb{E}[r_t(p_t^*)-r_t(p_t)].
\end{split}
\]
Obviously, the regret of the first summation can be bounded above by $2\overline p \bar t$. We focus on the second summation. For arbitrary $t> 2\bar t$,
\[
\begin{split}
\mathbb{E}[r_t(p_t^*)-r_t(p_t)]
=&\mathbb{E}[(r_t(p_t^*)-r_t(p_t))\pmb{1}(\mathcal{E}_t)]+\mathbb{E}[(r_t(p_t^*)-r_t(p_t))\pmb{1}(\bar{\mathcal{E}}_t)]\\
\leq & \mathbb{E}[(p_t^*\mu(\alpha_{i_t}'x_t+\beta_{i_t}p_t^*)-p_t\mu(\alpha_{i_t}'x_t+\beta_{i_t}p_t))\pmb{1}(\mathcal{E}_t)]+\frac{10\overline pn}{t} \\
= &\mathbb{E}[(|2\beta_{i_t}\dot{\mu}(\alpha_{i_t}'x_t+\beta_{i_t}\bar p_t)+\beta_{i_t}^2\bar{p}_t\ddot{\mu}(\alpha_{i_t}'x_t+\beta_{i_t}\bar p_t)|(p_t^*-p_t)^2)\pmb{1}(\mathcal{E}_t)]+\frac{10\overline pn}{t}  \\
\leq &\mathbb{E}[(\tilde{L}_2(p_t^*-\tilde p_t-\Delta_t)^2)\pmb{1}(\mathcal{E}_t)]+\frac{10\overline pn}{t}\\
\leq &2\tilde{L}_2L_0^2\mathbb{E}[||\tilde{\theta}_{\hat{\mathcal{N}}_t,t-1}-\theta_{i_t}||_2^2\pmb{1}(\mathcal{E}_t)]+4\tilde{L}_2\mathbb{E}[\Delta_t^2\pmb{1}(\mathcal{E}_t)]+\frac{10\overline pn}{t}\\
=  &2\tilde{L}_2L_0^2\mathbb{E}[||\tilde{\theta}_{j_t,t-1}-\theta_{j_t}||_2^2\pmb{1}(\mathcal{E}_t)]+4\tilde{L}_2\mathbb{E}[\Delta_t^2\pmb{1}(\mathcal{E}_t)]+\frac{10\overline pn}{t}\\
\leq  & 2\tilde{L}_2L_0^2\mathbb{E}[\tilde{B}_{j_t,t-1}^2\pmb{1}(\mathcal{E}_t)] +4\tilde{L}_2\mathbb{E}[\Delta_t^2\pmb{1}(\mathcal{E}_t)]+\frac{10\overline pn}{t},\\
\end{split}
\]
where the first inequality is from the probability of $\bar{\mathcal{E}}_t$,
the second equality is by applying Taylor's theorem (where $\bar p_t$ is some price between $p_t^*$ and $p_t$) with Assumption A-\ref{assumption:uniq_max} and Assumption A-\ref{assumption:mu}, the second inequality is from Assumption A-\ref{assumption:mu} and $\tilde L_2$ is some constant depending on $L,L_1,L_2,\overline p$, and both the last
 equality and the last inequality are from the definition of $\mathcal{E}_t$ (i.e., events $\mathcal{E}_{N,t}$ and $\mathcal{E}_{B,t}$).
Therefore, we have
\begin{equation}\label{eq:reg_geq_underline_t}
\mathbb{E}[r_t(p_t^*)-r_t(p_t)]\leq 2\tilde{L}_2L_0^2\mathbb{E}[\tilde{B}_{j_t,t-1}^2\pmb{1}(\mathcal{E}_t)] +4\tilde{L}_2\mathbb{E}[\Delta_t^2\pmb{1}(\mathcal{E}_t)]+\frac{10\overline pn}{t}.
\end{equation}
Summing over $t$, the sum of the last terms above obviously lead to the regret $O(n\log T)$. For the rest, we have
\[
\begin{split}
\sum_{t> 2\bar t}\mathbb{E}[\tilde{B}_{j_t,t-1}^2\pmb{1}(\mathcal{E}_t)]&\leq \frac{k_2 d \log T}{\Delta_0^2} \sum_{t> 2\bar t}\mathbb{E}\left[\frac{1}{\sqrt{\tilde q_{j_t}t}}\right]= \frac{k_2 d \log T}{\Delta_0^2} \sum_{t> 2\bar t}\sum_{j\in[m]}\sqrt{\frac{\tilde q_j}{t}} \\
&\leq \frac{k_2 d \log T}{\Delta_0^2}\sum_{j\in[m]}\sqrt{\tilde q_j T}\leq \frac{k_2 d \log T}{\Delta_0^2}\sqrt{ m T}
\end{split}
\]
for some constant $k_2$, 
where the first inequality is from $\mathcal{E}_t$ (i.e., $\mathcal{E}_{V,t}$) and the definition of $\tilde{B}_{j_t,t}^2$, the equality is by conditioning on $j_t=j$ for all $j\in[m]$, and the last inequality is because $\sum_j\tilde q_j=1$ and apply Cauchy-Schwarz. Hence
\begin{equation}\label{eq:sum_tilde_B}
\quad\sum_{t> 2\bar t}\mathbb{E}[\tilde{B}_{j_t,t-1}^2\pmb{1}(\mathcal{E}_t)]\leq  \frac{k_2 d \log T}{\Delta_0^2}\sqrt{ m T}.
\end{equation}

On the other hand, because $\hat{\mathcal{N}}_t=\mathcal{N}_{i_t}$ for all $t> 2\bar t$ on $\mathcal{E}_t$, 
\begin{equation}\label{eq:sum_delta_t}
\sum_{t> 2\bar t}\mathbb{E}[\Delta_t^2\pmb{1}(\mathcal{E}_t)]\leq \sum_{j\in[m]}\mathbb{E}\left[\sum_{t\in\tilde{\mathcal{T}}_{j,T}}\frac{\Delta_0^2}{\sqrt{\tilde T_{j,t}}}\right]  \leq \Delta_0^2\sum_{j\in[m]}\mathbb{E}\left[\sqrt{\tilde{T}_{j,T}}\right]\leq \Delta_0^2\sqrt{mT}.
\end{equation}
Putting (\ref{eq:reg_geq_underline_t}), (\ref{eq:sum_tilde_B}), and (\ref{eq:sum_delta_t}) together, we have
\[
\sum_{t> 2\bar t}\mathbb{E}[(r_t(p_t^*)-r_t(p_t))]\leq c_5d\log(T)\sqrt{mT}+c_5n\log T
\]
for some constant $c_5$, and together with the regret for $t<2\bar t$, we are done with the regret upper bound.



In the rest of this subsection, we prove the lemmas used in the proof of Theorem \ref{thm:main}.

\begin{lemma}\label{lemma:len_tilde_TjT}
For each $j\in[m]$ and $t\in\mathcal{T}$, with probability at least $1-\Delta$, $\tilde T_{j,t}\in [\tilde q_j t-\tilde D(t),\tilde q_j t+\tilde D(t)]$ for all $j\in[m]$, $t\in\mathcal{T}$, where $\tilde D(t)=\sqrt{t\log(2/\Delta)}$.
\end{lemma}
\textit{Proof. }
Obviously $\tilde T_{j,t}$ is a binomial random variable with parameter $t$ and $\tilde q_j$. Then we simply use Hoeffding inequality applied on sequence of i.i.d. Bernoulli random variable and a simple union bound on all $j\in[m]$ and $t\in\mathcal{T}$.
$\square$

\begin{lemma}\label{lemma:single_est_error}
For any $i\in[n]$ and $t\in\mathcal{T}$, 
let $V_{i,t}= I+\sum_{s\in\mathcal{T}_{i,t}} u_su_s'$, we have that
\[
||\hat{\theta}_{i,t}-\theta_i||_{V_{i,t}}\leq \frac{2\sqrt{(d+2)\log(1+T_{i,t}R^2/(d+2))+2\log(1/\Delta )}+2l_1 L}{l_1}
\]
with probability at least $1-\Delta$.
\end{lemma}
\textit{Proof. }
We first fix some $i\in[n]$, and we drop the index dependency on $i$ for convenience of notation.
At round $s$, the gradient of likelihood function $\nabla l_s(\phi)$ is equal to
\begin{equation}\label{eq:mu_gradient}
\nabla l_s(\phi)=(\mu(u_{s}'\phi)-d_{s})u_{s}.
\end{equation}
And its Hessian is
\begin{equation}\label{eq:mu_Hessian}
\begin{split}
\nabla^2 l_s(\phi)=&\dot{\mu}(u_{s}'\phi)u_su_s'.
\end{split}
\end{equation}

Applying Taylor's theorem, we obtain
\begin{equation}\label{eq:taylor1}
\begin{split}
0&\geq \sum_{s}l_{s}(\hat{\theta}_{t})-l_{s}(\theta)\\
&=\sum_{s}\nabla l_s(\theta)'(\hat{\theta}_{t}-\theta)+\frac{1}{2}\sum_s\dot{\mu}(u_s'\bar{\theta}_{t})(u_{s}'(\hat{\theta}_{t}-\theta))^2+\frac{l_1}{2} ||\hat{\theta}_t-\theta||_2^2-\frac{l_1}{2} ||\hat{\theta}_t-\theta||_2^2,\\
\end{split}
\end{equation}
where the first inequality is from the optimality of $\hat{\theta}_{t}$, and $\overline{\theta}_{t}$ is a point on line segment between $\hat{\theta}_{t}$ and $\theta$. Note that by our assumption and boundedness of $u_s$ and $\theta$, we have $\dot{\mu}(u_s'\bar{\theta}_{t})\geq l_1 $. Therefore, we have 
\begin{equation}\label{eq:delta_theta_lower}
\begin{split}
\sum_s\dot{\mu}(u_s'\bar{\theta}_{t})(u_{s}'(\hat{\theta}_{t}-\theta))^2+l_1||\hat{\theta}_t-\theta||_2^2\geq l_1||\hat{\theta}_{t}-\theta||_{V_t}^2,
\end{split}
\end{equation}
where $V_t= I+\sum_s u_su_s'$. 
On the other hand, we have
\begin{equation}\label{eq:nabla_L}
\begin{split}
\nabla l_s(\theta_i)=& -\epsilon_{s}u_{s},
\end{split}
\end{equation}
where $\epsilon_s$ is the zero-mean error, which is obviously sub-Gaussian with parameter $1$ as it is bounded. 

Now combining (\ref{eq:taylor1}), (\ref{eq:delta_theta_lower}), and (\ref{eq:nabla_L}), we have
\begin{equation}\label{eq:UCB_confidence_1}
\begin{split}
\frac{l_1}{2}||\hat{\theta}_t-{\theta}||_{V_t}^2
\leq & \sum_s\epsilon_{s}u_{s}'(\hat{\theta}_t-{\theta})+2l_1 L^2
\leq  ||\hat{\theta}_{t}-{\theta}||_{V_{t}}||Z_{t}||_{V_{t}^{-1}}+2l_1 L^2,\\
\end{split}
\end{equation}
where $Z_{t}:=\sum_s\epsilon_{s}{u}_{s}$, 
and the second inequality is from Cauchy-Schwarz and $||\hat{\theta}_t-\theta||_2\leq 2L$.  This leads to
$
||\hat{\theta}_t-{\theta}||_{V_t}\leq \frac{2}{l_1}||Z_t||_{V_t^{-1}}+2L.
$

To bound $||Z_t||_{V_t^{-1}}$, according to Theorem 1 in \citet{abbasi2011improved}, we have 
\[
||Z_t||_{V_t^{-1}}\leq \sqrt{(d+2)\log(1+\frac{T_{i,t}R^2}{d+2})+2\log(1/\Delta )}
\]
with probability at least $1-\Delta$ and we are done.
$\square$

\begin{corollary}\label{lemma:cluster_est_error}
For any $j\in[m]$ and $t\in\mathcal{T}$,
let $\tilde{V}_{j,t}:=I+\sum_{s\in\tilde{\mathcal{T}}_{j,t}} u_su_s'$, we have that
\[
||\tilde{\theta}_{j,t}-\theta_j||_{\tilde{V}_{j,t}}\leq \frac{2\sqrt{(d+2)\log(1+\tilde T_{j,t}R^2/(d+2))+2\log(1/\Delta )}+2l_1 L}{l_1}
\]
with probability at least $1-\Delta$.
\end{corollary}

Next result is the minimum eigenvalue of the empirical Fisher's information matrix.

\begin{lemma}\label{lemma:emp_fisher_info}
Denote $u_t' = (\tilde p_t+\Delta_t, x_t')$. 
For any $i\in[n]$ and 
\[
t> \max\left\{\left(\frac{8R\log((d+2)T)}{\lambda_1\Delta_0^2  \min_{i\in[n]}q_i}\right)^2,\left(\frac{\Delta_0^2}{c_0}\right)^2,\frac{2t_0}{\min_{i\in [n]}q_i}\right\},
\]
where $R:=2+\bar p^2$ and $\lambda_1:={1}/({(1+\bar p/c_0)^2+1})$, 
we have
\[
\mathbb{P}\bigg(\lambda_{\min}\Big(\sum_{s\in\mathcal{T}_{i,t}}u_su_s'\Big)< \frac{\lambda_1 \Delta_0^2 q_i\sqrt{t}}{2} \bigg)< \frac{1}{t^2}.
\]
\end{lemma}
\textit{Proof. }
Define
\[
M_{s}:=\mathbb{E}[\pmb{1}(i_s=i)u_su_s'|\mathcal{F}_{s-1}],
\]
and
\[
M:=\sum_{s=1}^{t}\mathbb{E}[\pmb{1}(i_s=i)u_su_s'|\mathcal{F}_{s-1}]=\sum_{s=1}^{t}M_s
=\begin{bmatrix}
A & B\\
B' & C\\
\end{bmatrix}
\]
where
\[
\begin{split}
    A:=&\sum_{s=1}^{t} \mathbb{E}[\pmb{1}(i_s=i)(\tilde p_s^2+\Delta_s^2)|\mathcal{F}_{s-1}]\\
    B:=&\sum_{s=1}^{t}\mathbb{E}[\pmb{1}(i_s=i)\tilde p_sx_s'|\mathcal{F}_{s-1}]\\
    C:=&\sum_{s=1}^{t}\mathbb{E}[\pmb{1}(i_s=i)x_sx_s'|\mathcal{F}_{s-1}]
\end{split}
\]
According to Proposition 3 in \citet{walton2020perturbed}, we have
\[
\lambda_{\min}(M)\geq \frac{\lambda_{\min}(C)^2}{(\|B\|_2+\lambda_{\min}(C))^2+\lambda_{\min}(C)^2}\min\{\lambda_{\min}(A-BC^{-1}B'),\lambda_{\min}(C)\}.
\]
Now let us analyze each term individually. By Assumption A-\ref{assumption:stochastic}, $q_it\geq \lambda_{\min}(C)\geq c_0 q_it$. It is also not difficult to get $\|B\|_2\leq \bar pq_it$. All we let to show is the lower bound of $\lambda_{\min}(A-BC^{-1}B')$, which is summarized in the following claim.

\textbf{Claim:} $\lambda_{\min}(A-BC^{-1}B')=A-BC^{-1}B'\geq \sum_{s=1}^{t} \mathbb{E}[\pmb{1}(i_s=i)\Delta_s^2|\mathcal{F}_{s-1}]$

To prove this claim, let us define
\[
\tilde u_s:=\left(\tilde p_s,x_s\right).
\]
That is, $\tilde u_s$ is the same as $u_s$ except without price perturbation. Obviously, $\tilde M:=\sum_{s=1}^{t}\mathbb{E}[\pmb{1}(i_s=i)\tilde u_s\tilde u_s'|\mathcal{F}_{s-1}]$ satisfies $\tilde M\geq 0$. Moreover, by Schur complement, we have $\tilde A-BC^{-1}B'\geq 0$ where $\tilde A=\sum_{s=1}^{t} \mathbb{E}[\pmb{1}(i_s=i)\tilde p_s^2|\mathcal{F}_{s-1}]$. Since $A-BC^{-1}B'=\tilde A-BC^{-1}B'+\sum_{s=1}^{t} \mathbb{E}[\pmb{1}(i_s=i)\Delta_s^2|\mathcal{F}_{s-1}]$, we are done with the claim. 

Above all, we are able to show that 
\[
\lambda_{\min}(M)\geq \frac{1}{(1+\bar p/c_0)^2+1}\min\{c_0q_it,q_i\Delta_0^2\sqrt{t}\}\geq \frac{\Delta_0^2 q_i\sqrt{t}}{(1+\bar p/c_0)^2+1}=\lambda_1\Delta_0^2 q_i\sqrt{t}
\]
where the second inequality is because $t>(\Delta_0^2/c_0)^2$. 
Since
\[
\sum_{s\in\mathcal{T}_{i,t}}u_su_s'=\sum_{s=1}^{t}\mathbbm{1}(i_s=i)u_su_s',
\]
then we have that 
\[
\begin{split}
&\mathbb{P}\left(\lambda_{\min}(\sum_{s\in\mathcal{T}_{i,t}}u_su_s')< \frac{\lambda_1\Delta_0^2 q_i\sqrt{t}}{2} \right)\\
=  &\mathbb{P}\left(\lambda_{\min}(\sum_{s\in\mathcal{T}_{i,t}}u_su_s')< \frac{\lambda_1\Delta_0^2 q_i\sqrt{t}}{2}, \lambda_{\min}\left(\sum_{s=1}^{t}\mathbb{E}[\mathbbm{1}(i_s=i)u_su_s'|\mathcal{F}_{s-1}]\right)\geq\lambda_1\Delta_0^2 q_i\sqrt{t}\right)\\
\leq & (d+2)e^{-\frac{\lambda_1\Delta_0^2 q_i\sqrt{t}}{4R}},
\end{split}
\]
where the last inequality is from  
 Theorem 3.1 in \citet{tropp2011user} with $\zeta=1/2$.
 
So for any $i\in[n]$ and 
\[
t> \left(\frac{8R\log(T(d+2))}{\lambda_1\Delta_0^2  \min_{i\in[n]}q_i}\right)^2,
\]
we have the simple union bound over $i\in[n],t\in\mathcal{T}$, 
$
(d+2)\exp({-{\lambda_1\Delta_0^2 q_i\sqrt{t}}/({4R})})< {1}/{t^2},
$
and the proof is complete.
$\square$

Clearly, if we combine Lemma \ref{lemma:emp_fisher_info} and Lemma \ref{lemma:single_est_error}, for any $i\in[n]$, $t>\bar t_{1}$ where 
\begin{equation}\label{eq:t_1_min}
\bar t_{1}=\max\left\{\left(\frac{8R\log((d+2)T)}{\lambda_1\Delta_0^2  \min_{i\in[n]}q_i}\right)^2,\left(\frac{\Delta_0^2}{c_0}\right)^2,\frac{2t_0}{\min_{i\in [n]}q_i}\right\},
\end{equation}
we have that
\begin{eqnarray}
\label{eq:est_bound_1}
||\hat{\theta}_{i,t}-\theta_i||_2&\leq& \frac{2\sqrt{(d+2)\log(1+tR^2/(d+2))+2\log t^2}+2l_1 L}{l_1\sqrt{\lambda_{\min}(V_{i,t})}}\\
\nonumber
&\leq& \frac{\sqrt{c(d+2)\log (1+t)}}{\sqrt{\lambda_{\min}(V_{i,t})}}
=B_{i,t}
\end{eqnarray}
for some constant $c>20/l_1^2$, and
\begin{equation}\label{eq:est_bound_2}
B_{i,t}\leq \frac{\sqrt{2c(d+2)\log (1+t)}}{\Delta_0\sqrt{\lambda_1 q_i\sqrt{t}} }
\end{equation}
with probability at least $1-2/t^2$.

The next lemma states that when estimation errors are bounded, under certain conditions we have $\hat{\mathcal{N}}_t=\mathcal{N}_{i_t}$.

\begin{lemma}\label{lemma:hat_N_eq_N}
Suppose for all $i\in [n]$ it holds that $||\hat{\theta}_{i,t-1}-\theta_{i}||_2\leq {B}_{i,t-1}$ and ${B}_{i,t-1}< \gamma/4$. Then
\[
\hat{\mathcal{N}}_t=\mathcal{N}_{i_t}.
\]
\end{lemma}
\textit{Proof. }
First of all, for $i_1,i_2\in[n]$, if they belong to different clusters and ${B}_{i_1,t-1}+{B}_{i_2,t-1}<\gamma/2$, we must have $||\hat{\theta}_{i_1,t-1}-\hat{\theta}_{i_2,t-1}||_2>{B}_{i_1,t-1}+{B}_{i_2,t-1}$ because
\[
\begin{split}
\gamma\leq & ||\theta_{i_1}-\theta_{i_2}||_2
\leq  ||\theta_{i_1}-\hat\theta_{i_1,t-1}||_2+||\hat\theta_{i_1,t-1}-\hat\theta_{i_2,t-1}||_2+||\hat\theta_{i_2,t-1}-\theta_{i_2}||_2\\
\leq & {B}_{i_1,t-1}+||\hat\theta_{i_1,t-1}-\hat\theta_{i_2,t-1}||_2+{B}_{i_2,t-1}
<  \gamma/2+||\hat\theta_{i_1,t-1}-\hat\theta_{i_2,t-1}||_2, 
\end{split}
\]
which implies that $||\hat\theta_{i_1,t-1}-\hat\theta_{i_2,t-1}||_2>\gamma/2>{B}_{i_1,t-1}+{B}_{i_2,t-1}$.

On the other hand, if $||\hat{\theta}_{i_1,t-1}-\hat{\theta}_{i_2,t-1}||_2>{B}_{i_1,t-1}+{B}_{i_2,t-1}$, we must have $i_1,i_2$ belongs to different clusters because
\[
\begin{split}
{B}_{i_1,t-1}+{B}_{i_2,t-1}< & ||\hat{\theta}_{i_1,t-1}-\hat{\theta}_{i_2,t-1}||_2
\leq  ||\theta_{i_1}-\hat\theta_{i_1,t-1}||_2+||\hat\theta_{i_1,t-1}-\hat\theta_{i_2,t-1}||_2+||\hat\theta_{i_2,t-1}-\theta_{i_2}||_2\\
\leq & {B}_{i_1,t-1}+||\hat\theta_{i_1,t-1}-\hat\theta_{i_2,t-1}||_2+{B}_{i_2,t-1},\\
\end{split}
\]
which implies $||\hat\theta_{i_1,t-1}-\hat\theta_{i_2,t-1}||_2>0$, i.e., they belong to different clusters.

Therefore, if $i\in \hat{\mathcal{N}}_t$, i.e.,  $||\hat{\theta}_{i_t,t-1}-\hat{\theta}_{i,t-1}||\leq {B}_{i_t,t-1}+{B}_{i,t-1}$, we must have that $i\in \mathcal{N}_{i_t}$ as well or ${B}_{i_t,t-1}+{B}_{i,t-1}\geq \gamma/2$ (which is impossible by our assumption that $B_{i,t-1}<\gamma/4$). 

On the other hand, if $i\in \mathcal{N}_{i_t}$, then we must have $||\hat{\theta}_{i_t,t-1}-\hat{\theta}_{i,t-1}||\leq {B}_{i_t,t-1}+{B}_{i,t-1}$, which implies that $i\in \hat{\mathcal{N}}_{t}$ as well.

Above all, we have shown that $\hat{\mathcal{N}}_{i_t}=\mathcal{N}_{i_t}$.
$\square$

Note that given (\ref{eq:est_bound_1}) and (\ref{eq:est_bound_2}), we have that $B_{i,t-1}<\gamma/4$ for all $i$ if
\[
t> 1+\frac{k_1((d+2)\log(1+T) )^2 }{\gamma ^4\lambda_1^2\Delta_0^4 \min_{i\in[n]}q_i^2}
\]
for some constant $k_1$. Therefore, for each $t>\bar t$ where
\begin{equation}\label{eq:t_2_min}
\bar t=\max\left\{4\bar t_{1},1+\frac{k_1((d+2)\log(1+T) )^2 }{\gamma ^4\lambda_1^2\Delta_0^4 \min_{i\in[n]}q_i^2}\right\},
\end{equation}
and $\bar t_{1}$ is defined in (\ref{eq:t_1_min}), $\hat{\mathcal{N}}_t=\mathcal{N}_{i_t}$ with probability at least $1-2n/t^2$.

The next lemma shows that the clustered estimation will be quite accurate when most of the $\hat{\mathcal{N}}_t$ is actually equal to $\mathcal{N}_{i_t}$.

\begin{lemma}\label{lemma:cluster_emp_fisher_info}
For any $t$ such that 
$
t> 2 \bar t,
$
we have
\[
\mathbb{P}\left(\lambda_{\min}\left(\sum_{s\in\tilde{\mathcal{T}}_{j_t,t}}u_su_s'\right)< \frac{\lambda_1\Delta_0^2 \sqrt{\tilde{q}_{j_t}t }}{8} \right)< \frac{7n}{t},
\]
where $\bar t$ is defined in (\ref{eq:t_2_min}).
\end{lemma}
\textit{Proof. }
The proof is analogous to Lemma \ref{lemma:emp_fisher_info}. Let $\mathcal{E}_{N,t}$ be the event such that $\hat{\mathcal{N}}_t=\mathcal{N}_{i_t}$, and $\tilde{\mathcal{E}}_{j,t}$ be the event such that $\tilde T_{j,t}\leq 3\tilde q_j t/2$. From our previous analysis, we know that given $t>\bar t$, $\mathcal{E}_{N,t}$ holds with probability at least $1-2n/t^2$. Also, according to Lemma \ref{lemma:len_tilde_TjT}, event $\tilde{\mathcal{E}}_{j,t}$ holds with probability at least $1-1/t^2$ given $t\geq 8\log(2T)/\min_{j\in[m]}\tilde q_j^2$ (which is satisfied by taking $t>\bar t$).

On event $\tilde{\mathcal{E}}_{j,t}$ and $\mathcal{E}_{N,s}$ for all $s\in [t/2,t]$ (which holds with probability at least $1-6n/t$), we have
\[
\begin{split}
    \lambda_{\min}\Big(\sum_{s=1}^{t}\mathbb{E}[\mathbbm{1}(j_s=j)u_su_s'|\mathcal{F}_{s-1}]\Big)\geq& \lambda_{\min}\Big(\sum_{s=t/2}^{t}\mathbb{E}[\mathbbm{1}(j_s=j)u_su_s'|\mathcal{F}_{s-1}]\Big)\geq\sum_{s=t/2}^{t}\lambda_1\Delta_0^2\tilde{q}_j(\tilde{T}_{j,s})^{-1/2}\\
    \geq&\lambda_1\Delta_0^2\frac{\sqrt{\tilde q_jt}}{4}.
\end{split}
\]
by definition of $\tilde{q}_j$ following a similar procedure as in Lemma \ref{lemma:emp_fisher_info}. 

Therefore, we have for any $t>2\bar t$,
\[
\begin{split}
&\mathbb{P}\left(\lambda_{\min}\left(\sum_{s\in\tilde{\mathcal{T}}_{j_t,t}}u_su_s'\right)< \frac{\lambda_1\Delta_0^2  \sqrt{\tilde{q}_{j_t} t }}{8}  \right)\\
=&\sum_{j\in[m]}\mathbb{P}\left(\lambda_{\min}\left(\sum_{s\in\tilde{\mathcal{T}}_{j_t,t}}u_su_s'\right)< \frac{\lambda_1\Delta_0^2  \sqrt{\tilde{q}_{j_t} t }}{8}  \Bigg|j_t=j\right)\mathbb{P}(j_t=j)\\
=&\sum_{j\in[m]}\mathbb{P}\left(\lambda_{\min}\left(\sum_{s\in\tilde{\mathcal{T}}_{j,t}}u_su_s'\right)< \frac{\lambda_1\Delta_0^2  \sqrt{\tilde{q}_{j} t }}{8}  \right)\tilde q_j.\\
\end{split}
\]
For each $j\in[m]$, we have
\[
\begin{split}
    &\mathbb{P}\left(\lambda_{\min}\left(\sum_{s\in\tilde{\mathcal{T}}_{j,t}}u_su_s'\right)< \frac{\lambda_1\Delta_0^2  \sqrt{\tilde{q}_{j} t }}{8}  \right)\\
    \leq & \mathbb{P}\left(\lambda_{\min}\left(\sum_{s\in\tilde{\mathcal{T}}_{j,t}}u_su_s'\right)< \frac{\lambda_1\Delta_0^2  \sqrt{\tilde{q}_{j} t }}{8},\bigcup_{s\in [t/2,t]}(\mathcal{E}_{N,t}\cup\tilde{\mathcal{E}}_{j,t}) \right)+\frac{6n}{t}\\
=&\mathbb{P}\left(\lambda_{\min}\left(\sum_{s\in\tilde{\mathcal{T}}_{j,t}}u_su_s'\right)< \frac{\lambda_1 \Delta_0^2 \sqrt{\tilde{q}_j t }}{8},\lambda_{\min}\left(\sum_{s\in\tilde{\mathcal{T}}_{j,t}}\mathbb{E}[u_su_s'|\mathcal{F}_{s-1}]\right)\geq  \frac{\lambda_1 \Delta_0^2 \sqrt{\tilde{q}_j t } }{4},\bigcup_{s\in [t/2,t]}(\mathcal{E}_{N,t}\cup\tilde{\mathcal{E}}_{j,t}) \right)\\
&+\frac{6n}{
t}
\leq \frac{7n}{t},
\end{split}
\]
where the first inequality is from the probability of the complement of $\bigcup_{s\in [t/2,t]}(\mathcal{E}_{N,t}\cup\tilde{\mathcal{E}}_{j,t})$, and the last inequality is by Theorem 3.1 in \citet{tropp2011user}, and we take
\[
t> \left(\frac{8R\log(2(d+2)T)}{\lambda_1\Delta_0^2  \min_{j\in[m]}\sqrt{\tilde q_j}}\right)^2.
\]
Since $\bar t>\left({8R\log(2(d+2)T)}/({\lambda_1\Delta_0^2  \min_{j\in[m]}\sqrt{\tilde q_j}})\right)^2$ by definition, we complete the proof.
$\square$

{
\section{Different $\theta_i$ for the Same Cluster}\label{ap_sec:relax_cluster}
In this section we present the technical lemmas in proving the regret of the modified CSMP when parameters $\theta_i$ within the same cluster are different. Note that we assume $||\theta_{i_1}-\theta_{i_2}||_2\leq \gamma_0$ for any $i_1,i_2$ in any cluster $\mathcal{N}_j$.

The first result is a corollary of Lemma \ref{lemma:hat_N_eq_N}. 

\begin{corollary}\label{cor:hat_N_eq_N}
Suppose for all $i\in [n]$ it holds that $||\hat{\theta}_{i,t-1}-\theta_{i}||_2\leq {B}_{i,t-1}$ and ${B}_{i,t-1}\in(\gamma_0/2, \gamma/6)$. Then (with $\gamma>3\gamma_0$) we have that
$\hat{\mathcal{N}}_t=\mathcal{N}_{i_t}.
$
Moreover, if we only have ${B}_{i,t-1}<\gamma/6$, we have $\hat{\mathcal{N}}_t\subset \mathcal{N}_{i_t}$.
\end{corollary}
\textit{Proof. }
For the first part of the corollary, the proof is almost identical to Lemma \ref{lemma:hat_N_eq_N}. First of all, for $i_1,i_2\in[n]$, if they belong to different clusters and ${B}_{i_1,t-1}+{B}_{i_2,t-1}<\gamma/3$, we must have $||\hat{\theta}_{i_1,t-1}-\hat{\theta}_{i_2,t-1}||_2>2{B}_{i_1,t-1}+2{B}_{i_2,t-1}$ because
\[
\begin{split}
\gamma\leq & ||\theta_{i_1}-\theta_{i_2}||_2
\leq  ||\theta_{i_1}-\hat\theta_{i_1,t-1}||_2+||\hat\theta_{i_1,t-1}-\hat\theta_{i_2,t-1}||_2+||\hat\theta_{i_2,t-1}-\theta_{i_2}||_2\\
\leq & {B}_{i_1,t-1}+||\hat\theta_{i_1,t-1}-\hat\theta_{i_2,t-1}||_2+{B}_{i_2,t-1}
<  \gamma/3+||\hat\theta_{i_1,t-1}-\hat\theta_{i_2,t-1}||_2,
\end{split}
\]
which implies that $||\hat\theta_{i_1,t-1}-\hat\theta_{i_2,t-1}||_2>2\gamma/3>2{B}_{i_1,t-1}+2{B}_{i_2,t-1}$.

On the other hand, if $||\hat{\theta}_{i_1,t-1}-\hat{\theta}_{i_2,t-1}||_2>2{B}_{i_1,t-1}+2{B}_{i_2,t-1}$, we must have $i_1,i_2$ belongs to different clusters because
\[
\begin{split}
2{B}_{i_1,t-1}+2{B}_{i_2,t-1}< & ||\hat{\theta}_{i_1,t-1}-\hat{\theta}_{i_2,t-1}||_2
\leq  ||\theta_{i_1}-\hat\theta_{i_1,t-1}||_2+||\theta_{i_1,t-1}-\theta_{i_2,t-1}||_2+||\hat\theta_{i_2,t-1}-\theta_{i_2}||_2\\
\leq & {B}_{i_1,t-1}+||\theta_{i_1,t-1}-\theta_{i_2,t-1}||_2+{B}_{i_2,t-1}\,\
\end{split}
\]
which implies $||\theta_{i_1,t-1}-\theta_{i_2,t-1}||_2>{B}_{i_1,t-1}+{B}_{i_2,t-1}\geq \gamma_0$ (where the second inequality is because $B_{i,t-1}\geq \gamma_0/2$), i.e., they belong to different clusters.

Therefore, if $i\in \hat{\mathcal{N}}_t$, i.e., $||\hat{\theta}_{i_t,t-1}-\hat{\theta}_{i,t-1}||\leq 2{B}_{i_t,t-1}+2{B}_{i,t-1}$, we must have that $i\in \mathcal{N}_{i_t}$ as well or ${B}_{i_t,t-1}+{B}_{i,t-1}\geq \gamma/3$ (which is impossible by our assumption that $B_{i,t-1}<\gamma/6$). 

On the other hand, if $i\in \mathcal{N}_{i_t}$, then we must have $||\hat{\theta}_{i_t,t-1}-\hat{\theta}_{i,t-1}||\leq 2{B}_{i_t,t-1}+2{B}_{i,t-1}$, which implies that $i\in \hat{\mathcal{N}}_{t}$ as well.
Summarizing, we have shown that $\hat{\mathcal{N}}_{i_t}=\mathcal{N}_{i_t}$.

For the second part, suppose this is not true. That is, there is some $i\in\hat{\mathcal{N}}_t$ with $i\not\in\mathcal{N}_{i_t}$, which implies $\|\theta_{i_t}-\theta_i\|\geq \gamma$ and $||\hat\theta_{i_1,t-1}-\hat\theta_{i_2,t-1}||_2\leq 2{B}_{i_t,t-1}+2{B}_{i,t-1}$. Note that
\[
\begin{split}
&\gamma\leq||\theta_{i_t}-\theta_{i}||_2
\leq  ||\theta_{i_t}-\hat\theta_{i_t,t-1}||_2+||\hat\theta_{i_t,t-1}-\hat\theta_{i,t-1}||_2+||\hat\theta_{i,t-1}-\theta_{i}||_2\\
\leq & {B}_{i_t,t-1}+||\hat\theta_{i_t,t-1}-\hat\theta_{i,t-1}||_2+{B}_{i,t-1}
<  \gamma/3+||\hat\theta_{i_1,t-1}-\hat\theta_{i_2,t-1}||_2;
\end{split}
\]
Thus we have $||\hat\theta_{i_1,t-1}-\hat\theta_{i_2,t-1}||_2>2\gamma/3$ and we have ${B}_{i_t,t-1}+{B}_{i,t-1}>\gamma/3$, contradicting with ${B}_{i,t-1}< \gamma/6$ for all $i$.
$\square$

Suppose in some time period $t$, product $i_t$ is in some neighborhood $\hat{\mathcal{N}}_t$ which satisfies $||\theta_{i_1}-\theta_{i_2}||_2\leq \tilde\gamma_0$ with some constant $\tilde\gamma_0$ for any $i_1,i_2\in\hat{\mathcal{N}}_t$. Let $\tilde\theta_{i_t,t}$ denote the estimated parameter by clustering all data in neighborhood $\hat{\mathcal{N}}_t$. The next lemma  measures the confidence bound of $\tilde\theta_{i_t,t}$ compared with any true parameter $\tilde\theta_i\in\hat{\mathcal{N}}_t$. 

\begin{lemma}\label{lemma:dummy_error}
When $T_{i,t}\geq q_it/2$ for all $i\in [N]$, we have for any $i\in\hat{\mathcal{N}}_t$,
\[
||\tilde\theta_{i_t,t}-{\theta_i}||_2\leq \frac{2\sqrt{(d+2)\log\left(1+\frac{tR^2}{d+2}\right)+4\log t}}{l_1\sqrt{\lambda_{\min}(V_{\hat{\mathcal{N}_t}})}}+\frac{L_1R^2\tilde\gamma_0\tilde q_{\hat{\mathcal{N}}_t}t}{l_1{\lambda_{\min}(V_{\hat{\mathcal{N}_t}})}}+\frac{2L}{\sqrt{\lambda_{\min}(V_{\hat{\mathcal{N}_t}})}}
\]
with probability at least $1-O(1/t^2)$, where $V_{\hat{\mathcal{N}_t}}=I+\sum_{t\in\mathcal{T}_{\hat{\mathcal{N}}_t}}u_su_s'$ and $\tilde q_{\hat{\mathcal{N}}_t}=\sum_{i\in\hat{\mathcal{N}}_t}q_i$.
\end{lemma}
\textit{Proof. }
The proof is quite similar to Lemma \ref{lemma:single_est_error}. 
We drop the index $i_t$ for convenience.
Note that for an arbitrary parameter $\phi\in\Theta$, since $\tilde\theta_{t}$ is the MLE, we have
\begin{equation}\label{eq:SC_MLE}
    \begin{split}
0&\geq \sum_{s}l_{s}(\tilde\theta_{t})-\sum_s l_{s}(\phi)=\sum_{s}\nabla l_s(\phi)'(\tilde\theta_{t}-\phi)+\frac{1}{2}\sum_s\dot{\mu}(u_s'\bar{\phi}_{t})(u_{s}'(\tilde\theta_{t}-\phi))^2\\
&+\frac{l_1}{2} ||\tilde\theta_{t}-\phi||_2^2-\frac{l_1}{2} ||\tilde\theta_{t}-\phi||_2^2\geq  \sum_{s}\nabla l_s(\phi)'(\tilde\theta_{t}-\phi)+\frac{l_1}{2}||\tilde\theta_{t}-{\phi}||_{V_{\hat{\mathcal{N}_t}}}^2-2l_1L^2,
\end{split}
\end{equation}
where the first inequality is from the optimality of $\tilde\theta_{t}$, and $\overline{\phi}_{t}$ is a point on line segment between $\tilde\theta_{t}$ and $\phi$.

Now we consider $\nabla l_s(\phi)$. By Taylor's theorem,
$
\nabla l_s(\phi)=\nabla l_s(\theta_s)+\nabla^2 l_s(\check\theta_s)'(\phi-\theta_s),
$
where $\theta_s$ is the true parameter at time $s$, and $\check{\theta}_s$ is a point between $\phi$ and $\theta_s$. As a result,
\begin{equation}\label{eq:SC_nabla}
    \nabla l_s(\phi)=-\epsilon_s u_s+\dot{\mu}(u_s'\check\theta_s)u_su_s'(\phi-\theta_s).
\end{equation}
Since $\phi\in\Theta$ is an arbitrary vector, we can let $\phi=\theta_i$ for any $i\in\mathcal{N}_j$. Combining (\ref{eq:SC_MLE}) and (\ref{eq:SC_nabla}), we have that with probability at least $1-1/t^2$.
\[
\begin{split}
    \frac{l_1}{2}||\tilde\theta_{t}-{\theta_i}||_{V_{\hat{\mathcal{N}_t}}}^2\leq & \sum_s\epsilon_su_s'(\tilde\theta_{t}-\theta_i)-\sum_s\dot{\mu}(u_s'\check\theta_s)(\theta_i-\theta_s)'u_su_s'(\tilde\theta_{t}-\phi)+2l_1L^2\\
    \leq & ||\sum_s\epsilon_su_s||_{V_{\hat{\mathcal{N}_t}}^{-1}}||\tilde\theta_{t}-{\theta_i}||_{V_{\hat{\mathcal{N}_t}}}+\sum_s ||\dot{\mu}(u_s'\check\theta_s)u_su_s'(\theta_i-\theta_s)||_{V_{\hat{\mathcal{N}_t}}^{-1}}||\tilde\theta_{t}-{\theta_i}||_{V_{\hat{\mathcal{N}_t}}}+2l_1L^2\\
    \leq & \sqrt{(d+2)\log\left(1+\frac{tR^2}{d+2}\right)+4\log t}||\tilde\theta_{t}-{\theta_i}||_{V_{\hat{\mathcal{N}_t}}}\\
    &+ \frac{\sum_s||\dot{\mu}(u_s'\check\theta_s)u_su_s'(\theta_i-\theta_s)||_2||\tilde\theta_{t}-{\theta_i}||_{V_{\hat{\mathcal{N}_t}}}}{\sqrt{\lambda_{\min}(V_{\hat{\mathcal{N}_t}})}}+2l_1L^2\\
    \leq &  \sqrt{(d+2)\log\left(1+\frac{tR^2}{d+2}\right)+4\log t}||\tilde\theta_{t}-{\theta_i}||_{V_{\hat{\mathcal{N}_t}}}+\frac{L_1R^2\tilde\gamma_0\tilde q_{\hat{\mathcal{N}}_t}t||\tilde\theta_{t}-{\theta_i}||_{V_{\hat{\mathcal{N}_t}}}}{2\sqrt{\lambda_{\min}(V_{\hat{\mathcal{N}_t}})}}+2l_1L^2,
\end{split}
\]
where the second inequality is from Theorem 1 in \citet{abbasi2011improved} and the last inequality is because ${T}_{i,t}\geq  q_{i}t/2$. By some simple algebra, above inequality implies that
\[
||\tilde\theta_{t}-{\theta_i}||_{V_{\hat{\mathcal{N}_t}}}\leq \frac{2\sqrt{(d+2)\log\left(1+\frac{tR^2}{d+2}\right)+4\log t}}{l_1}+\frac{L_1R^2\tilde\gamma_0\tilde q_{\hat{\mathcal{N}}_t}t}{l_1\sqrt{\lambda_{\min}(V_{\hat{\mathcal{N}_t}})}}+2L.
\]
This inequality further implies that
\begin{equation*}\label{eq:SC_final1}
    ||\tilde\theta_{t}-{\theta_i}||_2\leq \frac{2\sqrt{(d+2)\log\left(1+\frac{tR^2}{d+2}\right)+4\log t}}{l_1\sqrt{\lambda_{\min}(V_{\hat{\mathcal{N}_t}})}}+\frac{L_1R^2\tilde\gamma_0\tilde q_{\hat{\mathcal{N}}_t}t}{l_1{\lambda_{\min}(V_{\hat{\mathcal{N}_t}})}}+\frac{2L}{\sqrt{\lambda_{\min}(V_{\hat{\mathcal{N}_t}})}}
\end{equation*}
and we are done.
$\square$

The previous lemma shows that the estimation error is basically dependent on the value of $\lambda_{\min}(V_{\hat{\mathcal{N}_t}})$, which is described by the following lemma.
\begin{lemma}\label{lemma:min_eig_hatNt}
The value of $\lambda_{\min}(V_{\hat{\mathcal{N}_t}})$ satisfies the following.
\begin{enumerate}[(a)]
    \item If $t\geq \Omega(\bar t)$ and $t\leq k_3/(\max_i q_i^2 \gamma_0^4)$ for some constant $k_3$ (so that $k_3/(\max_i q_i^2 \gamma_0^4)\geq\Omega(\bar t)$ without loss of generality), $\lambda_{\min}(V_{\hat{\mathcal{N}_t}})\geq {\lambda_1\Delta_0^2 \sqrt{\tilde{q}_{j_t}t }}/{8}$ with probability at least $1-O(n/t)$.
    
    \item If $t\geq k_3/(\max_i q_i^2 \gamma_0^4)$, $\lambda_{\min}(V_{\hat{\mathcal{N}_t}})\geq {\lambda_1\Delta_0^2\sqrt{\tilde{q}_{\hat{\mathcal{N}}_t}/\tilde{q}_{j_t}} \sqrt{\tilde{q}_{\hat{\mathcal{N}}_t}t }}/{8}$ with probability at least $1-O(n/t)$.
\end{enumerate}
\end{lemma}
\textit{Proof.}
For part (a), it follows from the same procedure as Lemma \ref{lemma:cluster_emp_fisher_info}. The reason we want $t\leq O(1/(\max_i q_i^2 \gamma_0^4))$ is to guarantee $B_{i,s}> \gamma_0/2$ for $s\leq t$ so that we will have $\hat{\mathcal{N}}_s=\mathcal{N}_{i_s}$ by Corollary \ref{cor:hat_N_eq_N}. 

For part (b), with probability at least $1-O(n/t)$, we have $\hat{\mathcal{N}}_s\subset\mathcal{N}_{i_s}$ for any $s\geq k_3/(\max_i q_i^2 \gamma_0^4)$ according to Corollary \ref{cor:hat_N_eq_N}. Thus $\tilde T_{\hat{N}_s,s}\in [T_{i_s,s},\tilde T_{j_s,s}]$ and following the proof of Lemma \ref{lemma:cluster_emp_fisher_info}, with probability at least $1-O(n/t)$, $\lambda_{\min}(V_{\hat{\mathcal{N}_t}})\geq {\lambda_1\Delta_0^2\sqrt{\tilde{q}_{\hat{\mathcal{N}}_t}/\tilde{q}_{j_t}} \sqrt{\tilde{q}_{\hat{\mathcal{N}}_t}t }}/{8}$.
$\square$

The implication of the previous lemmas is the following. When $t\geq\Omega(\bar t)$ and $t\leq k_3/\max_i q_i^2 \gamma_0^4$, we have most of the time $\hat{\mathcal{N}}_t=\mathcal{N}_{i_t}$, and thus everything basically resembles the main setting of this paper. However, as $t$ keeps growing, we start to have only $\hat{\mathcal{N}}_t\subset\mathcal{N}_{i_t}$ according to Corollary \ref{cor:hat_N_eq_N}. That is, the $n$ products are no longer clustered into the $m$ clusters (with high probability) as we want. Therefore, the regret after $t\geq k_3/\max_i q_i^2 \gamma_0^4$ has to be analyzed more carefully. 
Now we provide the proof (sketch) of the theorem of regret of modified algorithm.

\begin{theorem}\label{thm:mod_CSMP}
The expected 
regret of the modified algorithm of CSMP is
\[ R(T)=
O\left(\frac{d^2 \log^2(dT)}{\min_{i\in[n]}q_i^2 }+d\log{T}\sqrt{\tilde m(T)T}+\Gamma(T)\right)
\]
where $\tilde m(T)$ and $\Gamma(T)$ are functions of $T$. In particular, when $T\leq k_3/\max_i q_i^2 \gamma_0^4$, we have $\tilde m(T)=m,\Gamma(T)=\min\{\gamma_0^{2}\sum_j\tilde q_j^2 T^2,T\}$; as $T\rightarrow\infty$, $\tilde m\rightarrow n$, $\Gamma(T)\rightarrow \bar\Gamma$ where $\bar\Gamma$ is a constant depending on the minimum gap between $\theta_{i_1},\theta_{i_2}$ within any of the same neighborhood.
\end{theorem}
\textit{Sketch of the Proof. }
Note that in this proof, we will calculate everything on all nice events (e.g., $||\hat\theta_{i,t}-\theta_i||_2\leq B_{i,t-1}$ for all $i\in\mathcal{N}$) as in the proof of Theorem \ref{thm:main} which hold with high probability, as the regret on their complement can be controlled to at most $O(n)$. Now let the results in Lemma \ref{lemma:dummy_error} hold. If $T\leq k_3/\max_i q_i^2 \gamma_0^4$, Corollary \ref{cor:hat_N_eq_N} shows that $\hat{\mathcal{N}}_t=\mathcal{N}_{i_t}$ for all $t\geq \Omega(\bar t)$. Thus,
combining part (a) of Lemma \ref{lemma:min_eig_hatNt} and Lemma \ref{lemma:dummy_error}, 
\begin{equation}\label{eq:ext_ineq}
    ||\theta_{i_t}-\tilde{\theta}_{j_t,t}||_{2}\leq O(\sqrt{d\log T}/(\tilde q_{j_t}t)^{1/4}+\min\{\gamma_0\sqrt{\tilde q_{j_t}t},1\}),
\end{equation}
where the part $\min\{\gamma_0\sqrt{\tilde q_{j_t}t},1\}$ is because any estimated parameter is bounded. 
Thus, to bound the regret when $T\leq k_3/\max_i q_i^2 \gamma_0^4$, the proof is almost identical to Theorem \ref{thm:main} so we neglect most part of the proof. We want to bound
$
r_t(p^*_t)-r_t(p_t)=O(r_t(p^*_t)-r_t(p_t')+\Delta_t^2).
$
Note that $\Delta_t^2=O\left(\tilde{T}_{\Hat{\mathcal{N}}_t,t}^{-1/2}\right)$, thus for the part of regret $\sum_t O\left(\tilde{T}_{\Hat{\mathcal{N}}_t,t}^{-1/2}\right)$, it is bounded as in Theorem \ref{thm:main}. 

To bound $r_t(p^*_t)-r_t(p_t')$, note that we have
$
r_t(p^*_t)-r_t(p_t')\leq 
O\left(
{||\theta_{i_t}-\tilde{\theta}_{j_t,t}||_{2}^2}
\right).
$
Thus combining (\ref{eq:ext_ineq}) and sum over $t$, we have this part of the regret is at most $O(d\log T\sqrt{mT}+\min\{\gamma_0^2 \sum_j\tilde q_j^2 T^2,T\})$. 

If $T>k_3/\max_i q_i^2 \gamma_0^4$, for any $t\geq k_3/\max_i q_i^2 \gamma_0^4$, since Corollary \ref{cor:hat_N_eq_N} shows that $\hat{\mathcal{N}}_t\subset\mathcal{N}_{i_t}$. Thus at any time $t$, we can take any subset of all estimated neighborhood of all $i$ (i.e., $\{\hat{\mathcal{N}}_{i,t}:i\in[n]\}$) whose union is equal to $\mathcal{N}$. Without loss of generality, let $\tilde m_t$ denote the number of such neighborhoods (denoted by $\{\hat{\mathcal{N}}_{[k],t}:k\in[\tilde m_t]\}$) and $\tilde\gamma_{0,t}$ denote the maximum distance between any two parameters within any $\hat{\mathcal{N}}_{[k],t}$ (for instance, when all $\hat{\mathcal{N}}_{t}=\mathcal{N}_{i_t}$,  we have $\tilde m_t=m,\tilde\gamma_{0,t}=\gamma_0$). Obviously, we have $\tilde m_t\in[m,n]$ and $\tilde\gamma_{0,t}\leq\gamma_0$. Then we basically follow the same procedure as earlier, and the regret in each time $t\geq k_3/\max_i q_i^2 \gamma_0^4$ is at most 
\begin{equation}\label{eq:ext_per_period_reg}
\begin{split}
    &O\left(\mathbb{E}\left[d\log T\sum_{k=1}^{\tilde m_t}\sqrt{\tilde{q}_{j_{[k]}}/\tilde q_{\hat{\mathcal{N}}_{[k],t}}} \sqrt{\tilde q_{\hat{\mathcal{N}}_{[k],t}}/t}+\min\left\{\tilde\gamma_{0,t}^2{\sum_{k=1}^{\tilde m_t}\tilde q_{\hat{\mathcal{N}}_{[k],t}}^2t},1\right\}+\sum_{k=1}^{\tilde m_t}\sqrt{\tilde q_{\hat{\mathcal{N}}_{[k],t}}/t}\right]\right)\\
    =&O\left(\mathbb{E}\left[d\log T\sum_{k=1}^{\tilde m_t}\sqrt{\tilde{q}_{j_{[k]}}/t}+\min\left\{\tilde\gamma_{0,t}^2{\sum_{k=1}^{\tilde m_t}\tilde q_{\hat{\mathcal{N}}_{[k],t}}^2t},1\right\}+\sqrt{\tilde m_t/t}\right]\right)\\
\end{split}
\end{equation}
where $j_{[k]}$ denote the index of the true neighborhood that $\hat{\mathcal{N}}_{[k],t}\subset \mathcal{N}_{j_{[k]}}$, and the expectation is taken over the realization of all neighborhoods $\{\hat{\mathcal{N}}_{i,t}:i\in[n]\}$. Thus, we can choose the $\tilde m_t$ neighborhood so that $\sum_{k=1}^{\tilde m_t}\sqrt{\tilde{q}_{j_{[k]}}}$ is minimized, and denote $\tilde m$ as a number so that $\sqrt{\tilde m}\geq \mathbb{E}\Big[\sum_{k=1}^{\tilde m_t}\sqrt{\tilde{q}_{j_{[k]}}}\Big]$ and $\tilde m\geq\mathbb{E}[\tilde m_t]$ for all $t$. Thus (\ref{eq:ext_per_period_reg}) is bounded above as 
\[
O\left(d\log T\sqrt{\tilde m/t}+\mathbb{E}\left[\min\left\{\tilde\gamma_{0,t}^2{\sum_{k=1}^{\tilde m_t}\tilde q_{\hat{\mathcal{N}}_{[k],t}}^2t},1\right\}\right]+\sqrt{\tilde m/t}\right),
\]
and we are done with the expression of regret by summing over all $t$. Note that as $T\rightarrow\infty$, it is obvious that each $\hat{\mathcal{N}}_{i,t}$ becomes $\{i\}$ itself when $B_{i,t-1}$ is sufficiently small compared with the minimum gap between any two parameters within the same neighborhood. Thus $\tilde\gamma_{0,t}$ becomes 0 (implying $\Gamma(T)$ is bounded) and $\tilde m\rightarrow n$ as $\tilde m_t\rightarrow n$. 
$\square$}

{
\section{Price Dependency of the Real Dataset}\label{app_sec:price_dependency}
In this section, we conduct a data analysis of the real dataset in Section \ref{subsec:real_simulation}. The purpose is to show that the demand of each product mainly depends on its own price. 
For illustration of the effect of price, we exclude the features and adopt a simple logistic regression model for each product, which is defined as, if we make the dependency on index $i$ implicit, 
\[
\mathbb{E}[d_t(p)]=\frac{1}{1+e^{-\alpha-\beta p}}
\]
where $p$ is the price of interest, which can be the product's own price or any other product's price.

\begin{figure}[t]
    \centering
    \includegraphics[scale=0.7]{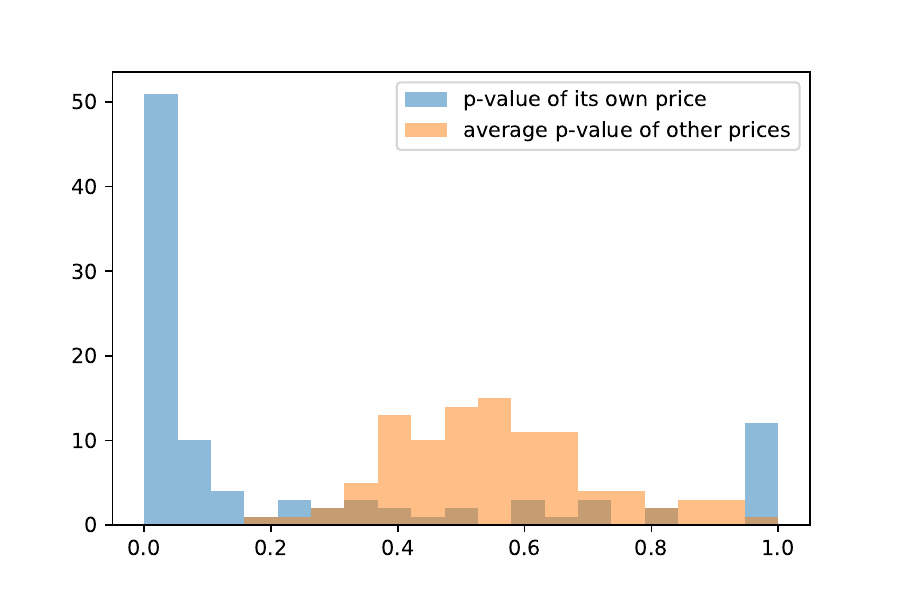}
    \caption{P-value of own price versus other prices.}
    \label{fig:p_value}
\end{figure}
 To test the significance of each price, we evaluate the p-value of the hypothesis test with 
\[
H_0:\beta=0\quad \text{VS} \quad H_a:\beta\neq 0.
\]
First, for each product $i$, we calculate the p-value of its own price and the average p-value of other prices, and results are summarized in Figure \ref{fig:p_value}, which is a histogram of the two p-values of all products. From this histogram, we can clearly see that most products have significantly lower p-value of its own price than other prices, showing that the demand is mainly dependent on its own price.

Of course, this experiment mainly shows that overall other products' prices do not have significant impact on each product's demand, but we still do not know how specifically price of product $i$ affects demand of product $i'$. Next, we will investigate one by one of each product's price on other products. For instance, fix product $i$, we calculate the p-value of price $p_{i,t}$ on the demand of any other product $i'\neq i$, and then count how many products $i'\neq i$ that price of product $i$ has significant (i.e., p-value$<$0.05) impact on. 

Table \ref{tab:significant} summarize this result. On average, the price of each product only significantly affects the demand of 9.44 other products, compared with the fact that number of products, whose demand is significantly affected by its own price, is equal to 51. Note that it is not surprising some products' demands are not significantly affected by their prices because of the data scarcity due to low sales and popularity. For the purpose of simulation in Section \ref{subsec:real_simulation}, we will still fit the data of these 100 low-sale products as it is for illustrative purposes. 
\begin{table}[h]
\centering
\begin{tabular}{l|llll}
\hline
\hline
                              & Mean & Standard Deviation & Maximum & Minimum \\
\hline
Number of significant p-value &  9.44    &     4.49               &    22     &    2   \\
\hline
\end{tabular}
\caption{Number of significant p-value on demand of other products.}
\label{tab:significant}
\end{table}

}

\end{document}